\documentclass[10pt,twocolumn,Letter]{IEEEtran}

\usepackage{graphicx}
\usepackage{amsmath}
\usepackage{amssymb}
\usepackage{amsthm,bm}
\usepackage{caption}
\usepackage{subcaption}
\usepackage{float}
\usepackage{optidef}
\usepackage[normalem]{ulem}
\usepackage{algorithm}
\usepackage{algpseudocode}
 \usepackage{url}
 \usepackage[outdir=./]{epstopdf}
 \usepackage{tabularx}
 \usepackage{multirow}

\newcommand{\figsize}{0.3}

\newtheorem{lemma}{Lemma}

\begin{document}
%
\title{Polytopic Matrix Factorization: Determinant Maximization Based Criterion and Identifiability}
%
%
%

\author{Gokcan Tatli,~\IEEEmembership{Member,~IEEE,}
and~Alper T. Erdogan,~\IEEEmembership{Senior Member,~IEEE,}
\thanks{ This work is partially supported by an AI Fellowship provided by the
KUIS AI Lab.}
}

%
%

\newcommand{\ba}{\mathbf{a}}
\newcommand{\be}{\mathbf{e}}
\newcommand{\bx}{\mathbf{x}}
\newcommand{\bu}{\mathbf{u}}
\newcommand{\bv}{\mathbf{v}}
\newcommand{\by}{\mathbf{y}}
\newcommand{\bq}{\mathbf{q}}
\newcommand{\bg}{\mathbf{g}}
\newcommand{\bY}{\mathbf{Y}}
\newcommand{\bX}{\mathbf{X}}
\newcommand{\bU}{\mathbf{U}}
\newcommand{\bV}{\mathbf{V}}
\newcommand{\bB}{\mathbf{B}}
\newcommand{\bs}{\mathbf{s}}
\newcommand{\bz}{\mathbf{z}}
\newcommand{\bS}{\mathbf{S}}
\newcommand{\bL}{\mathbf{L}}
\newcommand{\bK}{\mathbf{K}}
\newcommand{\bH}{\mathbf{H}}
\newcommand{\bw}{\mathbf{w}}
\newcommand{\bW}{\mathbf{W}}
\newcommand{\bA}{\mathbf{A}}
\newcommand{\bC}{\mathbf{C}}
\newcommand{\bI}{\mathbf{I}}
\newcommand{\bF}{\mathbf{F}}
\newcommand{\bQ}{\mathbf{Q}}
\newcommand{\Ball}{\mathcal{B}}
\newcommand{\Tin}{\mathcal{T}_{in}}
\newcommand{\Tout}{\mathcal{T}_{out}}
\newcommand{\Pcal}{\mathcal{P}}
\newcommand{\Pex}{\Pcal_{\text{ex}}}
\newcommand{\bD}{\mathbf{D}}
\newcommand{\bPi}{\bm{\Pi}}
\newcommand{\blambda}{\bm{\lambda}}
\newcommand{\bSigma}{\bm{\Sigma}}
\newcommand{\bGam}{\bm{\Gamma}}

\newtheorem{theorem}{Theorem}

\theoremstyle{definition}
\newtheorem{definition}{Definition}[section]

\markboth{Published in IEEE Transactions on Signal Processing, 2021,doi: 10.1109/TSP.2021.3112918.}{}%

\maketitle

\begin{abstract}
We introduce Polytopic Matrix Factorization (PMF) as a novel data decomposition approach. In this new framework, we model input data as unknown linear transformations of some latent vectors drawn from a polytope. In this sense, the article considers a semi-structured data model, in which the input matrix is modeled as the product of a full column rank matrix and a matrix containing samples from a polytope as its column vectors. The choice of polytope reflects the presumed features of the latent components and their mutual relationships. As the factorization criterion, we propose the determinant maximization (Det-Max) for the sample autocorrelation matrix of the latent vectors. We introduce a sufficient condition for identifiability, which requires that the convex hull of the latent vectors contains the maximum volume inscribed ellipsoid of the polytope with a particular tightness constraint. Based on the Det-Max criterion and the proposed identifiability condition, we show that all polytopes that satisfy a particular symmetry restriction qualify for the PMF framework. Having infinitely many polytope choices provides a form of flexibility in characterizing latent vectors. In particular, it is possible to define latent vectors with heterogeneous features, enabling the assignment of attributes such as nonnegativity and sparsity at the subvector level. The article offers examples illustrating the connection between polytope choices and the corresponding feature representations.
\end{abstract}


%
\IEEEpeerreviewmaketitle

\section{Introduction}

\IEEEPARstart{M}{atrix} factorization methods are fundamental algorithmic tools for both signal processing and machine learning (e.g., \cite{cichocki2009nonnegative, fu2019nonnegative, smaragdis2003non,xu2003document,lin2018maximum,erdogan2013class,georgiev2005sparse}). Revealing information hidden inside input data is a central problem in several applications. A common solution approach is to model the input matrix as the product of two factors. 

In unsupervised settings, both factors are unknown, and there is no available training information for their estimation. Structured matrix factorization (SMF) methods utilize prior information or assumptions on both factors, such as rank, nonnegativity, sparsity, and antisparsity, to achieve the desired decomposition. In the semi-structured matrix factorization that we pursue in this article, the left factor is simply a full column rank matrix with no additional structure. In this setting, we refer to the columns of the right-factor as latent vectors, which have some presumed structure. The left factor is the linear transformation matrix that maps latent vectors to inputs. Due to the full column rank assumption on the left-factor matrix, the scope of the article is limited to the (over)determined case.

We can define the attributes of latent vectors through the choice of their domain. The topology of this set determines both the individual properties of latent vector components and their relationships. For example, for Nonnegative Matrix Factorization (NMF)  \cite{paatero1994positive,lee1999learning,fu2019nonnegative}, the domain choice is the nonnegative orthant, and for a related approach, Simplex Structured Matrix Factorization (SSMF)   \cite{lin2018maximum}, it is the unit simplex. In addition, we can list two polytopic sets, namely the $\ell_\infty$-norm-ball for the antisparse version of Bounded Component Analysis (BCA)  \cite{cruces2010bounded,erdogan2013class} and the $\ell_1$-norm-ball for Sparse Component Analysis (SCA)  \cite{georgiev2005sparse,elad2010sparse,babatas2018algorithmic}, as further examples. 

These matrix factorization frameworks have found successful applications in different domains. Basic applications of NMF include document mining \cite{xu2003document}, feature extraction for natural images \cite{lee1999learning}, source separation for hyperspectral images \cite{yokoya2011coupled}, community detection \cite{huang2019detecting} and audio demixing \cite{smaragdis2003non,ozerov2009multichannel,leplat2020blind}. The main application area for SCA has been sparse dictionary learning, which has laid the foundation for the {\it sparse coding} principle, utilized in both computational neuroscience \cite{olshausen1997sparse}
and machine learning \cite{lee2007efficient}. BCA has both dependent source and short-data-length separation capabilities with applications in natural image separation and digital communications \cite{cruces2010bounded,erdogan2013class,babatas2018algorithmic}.

Identifiability is a crucial concept in determining the applicability of matrix factorization methods. It concerns the ability to obtain unique factors of the input data up to some acceptable ambiguities, such as sign and permutation. All of the aforementioned domain choices, i.e., the nonnegative orthant, unit simplex, $\ell_\infty$, and $\ell_1$-norm balls, have been shown to lead to identifiable data models \cite{donoho2003does, lin2015identifiability, erdogan2013class, babatas2018algorithmic}. 

A fundamental question is addressed in this article: \textit{ Can we extend domains enabling identifiability beyond these existing examples?} We indeed provide a positive answer and show that all polytopes that comply with a particular symmetry restriction qualify. We refer to the associated framework as Polytopic Matrix Factorization (PMF) and the polytopes that qualify for this framework as ``identifiable polytopes". The availability of infinite identifiable polytope choices offers a degree of freedom in generating a diverse set of feature attributes for latent vectors. For example, as shown in Section \ref{sec:specialPMF}, we can combine different attributes, such as nonnegativity, sparsity, and antisparsity, which separately exist in the SMF frameworks listed above, through proper selection of polytopes. Furthermore, as illustrated by the example in Section \ref{sec:numexperiments}, we can define latent vectors with heterogeneous features by performing attribute assignments at the subvector level. In other words, it is possible, for example, to define latent vectors in which only a fraction of the components are nonnegative, and the sparsity is imposed on subsets of components. Part of this work, mainly Theorem \ref{thm:generalized} in Section \ref{sec:generalizedpmf} on the characterization of identifiable polytopes for PMF, was presented in \cite{tatli:2021icassp}.

For the identification of the factor matrices in the PMF model, we propose the use of {\it determinant maximization criterion}, which has been successfully employed in both the NMF \cite{schachtner2011towards,fu2019nonnegative} and BCA \cite{erdogan2013class,babatas2018algorithmic} approaches. The determinant of the sample correlation matrix acts as a scattering measure for the set of latent vectors and its maximization targets to exploit the presumed spread of the latent vectors inside the polytope.

The success of the Det-Max criterion for the perfect recovery of factor matrices is dependent on the scattering of the latent vectors inside the polytope. Intuitively, if they are concentrated in a relatively small subregion of the polytope, they would fail to reflect its topology. Therefore, any criterion exploiting polytope membership information would fail in such a case. This article offers a sufficient condition on the spread of latent vectors inside the polytope to enable identifiability for the Det-Max criterion, providing theoretical grounds for this intuition.

We can position the proposed identifiability condition on PMF as an extension of the existing results in other SMF frameworks.
The existing BCA approaches use two particular polytopes: $\ell_\infty$-norm-ball for antisparse and $\ell_1$-norm-ball for sparse components. The identifiability results for BCA assume that latent vectors contain the vertices of these polytopes \cite{erdogan2013class,babatas2018algorithmic}. Similarly, the early identification results for NMF used the condition that latent vectors include the scaled corner points of the unit simplex  \cite{donoho2003does,laurberg2008theorems,arora2016computing}. This condition is referred as the separability/pure pixel condition or the groundedness assumption. These assumptions in both frameworks require the inclusion of specific points in a random collection of latent vectors, which is too restrictive for practical plausibility.
The ``sufficiently scattered" condition proposed for NMF in \cite{huang2013non} significantly relaxed the corner inclusion assumption. This new condition requires that the conic hull of latent vectors  contains a specific ``reference cone"  
\cite{huang2013non,fu2015blind,fu2018identifiability,fu2019nonnegative}. Lin \textit{et al.} \cite{lin2015identifiability} proposed a related sufficient scattering condition for SSMF.

Using geometric principles similar to those proposed for NMF and SSMF, we introduce a novel ``sufficiently scattered" condition for the PMF framework. 
This new condition is much weaker than the vertex inclusion assumption used in BCA identifiability analysis for $\ell_1$ and $\ell_\infty$-norm balls  \cite{erdogan2013class,babatas2018algorithmic}. Furthermore, it is applicable to the class of all identifiable polytopes. The proposed criterion uses  the maximum volume inscribed ellipsoid (MVIE) of the polytope, which can be considered its best inscribed ellipsoidal approximation. According to this new criterion, the samples should be sufficiently spread across the polytope such that their convex hull, i.e., the smallest convex set that contains these samples, also contains the MVIE of the polytope with a particular tightness constraint. In other words, they can be used to construct a more accurate model of the polytope than its best ellipsoidal approximation. In Section \ref{sec:specialPMF}, we illustrate latent vector sets satisfying the proposed identifiability condition for some particular polytopes. As demonstrated by these examples, this new condition leads to more practically plausible identifiable data models, compared to the vertex inclusion assumption. Furthermore, this condition forms the basis for the generalized polytope identifiability result offered in Section \ref{sec:generalizedpmf}. 
We note that Lin \textit{et al.} \cite{lin2018maximum} proposed {\it an algorithm for SSMF} based on the MVIE of {\it the convex hull of the input vectors}. In our context, we use the MVIE of {\it the polytope} to define a {\it sufficient condition} for the {\it PMF identifiability}. Therefore, the MVIE concept is used in different domains (input vs latent vector spaces) and for different purposes (algorithm vs identifiability analysis).

We can summarize the main contributions of the article as follows:
\begin{itemize}
    \item We offer a new, unsupervised data decomposition framework called Polytopic Matrix Factorization (PMF).
    \item We propose the use of a Det-Max criterion and introduce a novel geometric identifiability condition  for PMF based on the MVIE of the polytope.
    \item We provide a characterization of identifiable polytopes. 
    \item We illustrate the potential of the proposed PMF framework in terms of flexible description of latent vectors with heterogeneous features. 
\end{itemize}

The following describes the organization of this article. In Section \ref{sec:PMFsetup}, we provide the data model and the Det-Max optimization criterion for PMF. We also introduce the proposed sufficient scattering-based identifiability condition for the PMF framework. In Section \ref{sec:specialPMF}, we focus on four special polytopes corresponding to the combinations of antisparse/sparse and nonnegative/signed attributes and provide their identifiability results. In Section \ref{sec:generalizedpmf}, we offer a theorem on the characterization of polytopes that qualify for the PMF framework. Section \ref{sec:algorithm} presents a PMF algorithm adopted from the NMF literature. Section \ref{sec:numexperiments} contains numerical examples for PMF. Finally, Section \ref{sec:conclusion} concludes the study.

Table \ref{tab:notation} outlines the basic notations used throughout the article.
\begin{table}[ht]
\renewcommand{\arraystretch}{1.1}
\captionsetup{justification=centering, labelsep=newline}
\caption{\textsc{Notation}}
\label{tab:notation}
\begin{center}
 \begin{tabularx}
{0.45\textwidth}{|c | c |}
 \hline
 Notation & Meaning \\ [0.05ex] 
 \hline\hline
 $\bS_{j,:} \ (\bS_{:,j})$ & $j^{\textrm{th}}$ row (column) of matrix $\bS$
 \\ 
 \hline
  \multirow{2}{*}{$\be_k$} & Standard basis vector \\
 & which is all zeros except $1$ at the $k^{\textrm{th}}$ index \\
 \hline
  \multirow{2}{*}{$\langle \bx,\by \rangle$} & Euclidean inner product between the vectors\\ & $\bx,\by \in \mathbb{R}^n$ defined as $\by^T\bx$
 \\ 
 \hline
  \multirow{2}{*}{$\Ball_p$} & unit $\ell_p$-norm-ball defined as\\
             & $\{ \bx  \mid \|\bx\|_p\le 1\}$
 \\ [0.05ex] 
 \hline
    \multirow{2}{*}{$\mathcal{S}^{*,\mathbf{d}}$} & the polar of the set $\mathcal{S}$ with respect to the point $\mathbf{d}$\\
     & (See Appendix \ref{app:convex}) 
 \\ 
 \hline
 $\text{conv}(\bS)$ & the convex hull of the columns of $\bS$\\ 
 \hline
  $\text{cone}(\bS)$ & the conic hull of the columns of $\bS$  \\
  \hline
  \multirow{2}{*}{$\mathcal{K}^d$} & the dual cone of  the cone $\mathcal{K}$ which is defined as \\  & $\{\bx \mid \langle \bx, \by \rangle \ge 0, \forall \by \in \mathcal{K}\}$  \\
 \hline
 \multirow{2}{*}{$\text{ext}(\Pcal)$} & the set of extreme points (vertices, or corner points) \\ & of the polytope $\Pcal$  \\
 \hline
  $\text{bd}(\mathcal{S})$ & the boundary of the set  $\mathcal{S}$  \\
 \hline
  \multirow{2}{*}{$\bA(S)$} & the image of the set $S$ under \\ & the linear transformation with matrix $\bA$
 \\
 \hline
 \multirow{2}{*}{$\mathbf{1} (\mathbf{0})$} & a  vector or a matrix with all ones (zeros)\\ &  of the appropriate dimensions \\
 \hline
  $\mathbb{R}_+^r$ & The $r$-dimensional nonnegative orthant\\
  \hline
 $\bA\succ(\succeq)\mathbf{0}$ & $\bA$ is a positive (semi)definite matrix \\
 \hline
 \multirow{2}{*}{$\alpha \mathcal{S}+\mathbf{d}$} & The image of the set $\mathcal{S}$ under \\ & the transformation $f(\bx)=\alpha \bx+\mathbf{d}$\\
  \hline
\end{tabularx}
\end{center}
\end{table}

\section{Polytopic Matrix Factorization Problem}
\label{sec:PMFsetup}
In this section, we introduce Polytopic Matrix Factorization as a new unsupervised data decomposition framework.
We  start by describing the PMF problem in Section \ref{sec:PMFproblem}. Then, we define the determinant maximization based criterion for the PMF problem  in Section \ref{sec:crit}. In connection with this criterion, we provide  the proposed  PMF identifiability condition in Section \ref{sec:proposedidentifiability}.

\subsection{PMF Problem}
 \label{sec:PMFproblem}

For the PMF problem, we assume the following generative data model: the input matrix $\bY \in \mathbb{R}^{M\times N}$ is given by
\begin{eqnarray}
 \bY=\bH_g\bS_g, \label{eq:PMFgen}
\end{eqnarray}
where
\begin{itemize}
\item $\bH_g\in \mathbb{R}^{M \times r}$ is the ground truth of the left-factor matrix, which is assumed to be full column rank; and 
\item $\bS_g\in \mathbb{R}^{r \times N}$ is the ground truth of the right-factor matrix, where we assume $r\le \min(M,N)$.
The underlying assumption of the PMF framework can be written as
\begin{eqnarray}
 {\bS_g}_{:,j} \in \mathcal{P}, \; j= 1, \ldots, N, \label{eq:PMFgen2}
\end{eqnarray}
where $\Pcal$ is a convex polytope.
\item There are two canonical forms to describe $\Pcal$:
\begin{itemize}
    \item \uline{\bf H-Form} (\textit{Intersections of Half-spaces}): A convex polytope $\Pcal$ can be defined in the form
    \begin{eqnarray}
    \Pcal=\{\bx \ \mid \langle \ba_i, \bx\rangle \leq b_i, i=1, \ldots, f\}, \label{eq:ptope1}
    \end{eqnarray}
    where $f$ is the number of faces, and vectors $\ba_i$ are the face normals. Each inequality in (\ref{eq:ptope1}) represents a half-space, and the intersection of  these half-spaces forms a convex polyhedron. If $\Pcal$ is bounded, we refer to it as a convex polytope.
    \item \uline{\bf V-Form} (\textit{Convex Hull of Vertices}): Let $\bv_1, \ldots, \bv_m$ represent the vertices, or the extreme points, of a convex polytope, where $m$ is the number of vertices, and then we can use 
    \begin{eqnarray}
    \Pcal=\text{conv}(\left[\begin{array}{cccc} \bv_1 & \bv_2 &  \ldots & \bv_m\end{array}\right]),
    \label{eq:ptope2}
    \end{eqnarray}
    for the representation of the corresponding polytope.
\end{itemize}
The conversion between these  canonical forms is referred  to as the {\it polyhedral representation conversion problem}  \cite{bremner2009polyhedral}.
\end{itemize}

The goal of PMF is to obtain a factorization of the input data $\bY$ in the form
 $\bY=\bH\bS$
such that these factors satisfy
\begin{align}
 \bH&=\bH_g\bPi^T \bD^{-1}, \label{eq:desired1} \\ 
 \bS&=\bD \bPi\bS_g, \label{eq:desired2}
 \end{align}
where  $\bPi \in \mathbb{R}^{r \times r}$ is a permutation matrix that represents unresolvable ambiguity in obtaining the ordering of the columns (rows) of $\bH_g$ ($\bS_g$) and $\bD \in \mathbb{R}^{r \times r }$ is a full rank diagonal matrix that corresponds to the scaling ambiguity.

In Section \ref{sec:specialPMF}, we  provide PMF identifiability results using  four particular polytopes of practical interest:
\begin{itemize}
\item[i.] $\Pcal=\Ball_\infty$, i.e., the unit $\ell_\infty$-norm-ball, which we refer to as the ``antisparse" PMF case hereafter;
\item[ii.] $\Pcal=\Ball_1$, i.e., the unit $\ell_1$-norm-ball, which we refer to as the ``sparse" PMF case;
\item[iii.] $\Pcal=\Ball_\infty \cap \mathbb{R}^r_+=\{\bx \mid \mathbf{0}\le \bx \le \mathbf{1}\}$, which is referred to as the ``antisparse nonnegative" PMF case; and
\item[iv.] $\Pcal=\Ball_1 \cap \mathbb{R}^r_+=\{\bx \mid \mathbf{1}^T \bx \le 1, \bx\ge 0\}$, which is referred to as the ``sparse nonnegative" PMF case.
\end{itemize}
We  extend these results for a wider range (infinite number) of polytopes  in Section \ref{sec:generalizedpmf}.

\subsection{The Criterion for Identification}
 \label{sec:crit}
For the PMF problem outlined in Section \ref{sec:PMFsetup}, we propose the use of the determinant maximization (Det-Max) approach, which has been successfully utilized in both the NMF  \cite{chan2011simplex,fu2018identifiability,fu2019nonnegative} and antisparse/sparse BCA  \cite{erdogan2013class,inan2014convolutive,babatas2018algorithmic} frameworks.
Therefore, the following serves as the prototype optimization problem throughout the article:
\begin{maxi!}[l]<b>
{\bH\in \mathbb{R}^{M \times r},\bS\in \mathbb{R}^{r\times N}}
{\det(\bS\bS^T)\label{eq:detmaxobjective}}{\label{eq:detmaxoptimization}}{}
\addConstraint{\bY=\bH \bS}{\label{eq:detmaxconstr1}}{}
\addConstraint{\bS_{:,j} \in \mathcal{P}}{,\quad}{j= 1, \ldots, N.\label{eq:detmaxconstr2}}{ }
\end{maxi!}
The objective function of the Det-Max Optimization Problem in (\ref{eq:detmaxobjective}) is equal to the determinant of the (scaled) sample correlation matrix $\mathbf{R}_\mathbf{s}=\frac{1}{N}\bS\bS^T=\frac{1}{N}\sum_{j=1}^N\bS_{:,j}\bS_{:,j}^T$, which is a measure of nondegenerate scattering. In the zero mean case, the objective function boils down to ``generalized variance", defined as the product of the eigenvalues of the covariance matrix  \cite{sengupta2004generalized}. These eigenvalues are the variances for the principal directions. Due to its product form, generalized variance is sensitive to the existence of directions with small variations. Therefore, its maximization corresponds to a nondegenerate spreading of the corresponding samples in all directions.

The determinant minimization criterion employed in the NMF approaches of \cite{schachtner2011towards,fu2019nonnegative} corresponds to the minimization of det($\bH^T\bH$) in our problem. However, we utilize a dual approach that maximizes det($\bS\bS^T$) for our identifiability results.
The following definition classifies generative PMF settings with respect to the determinant maximization criterion:

\begin{definition}\label{def:detmaxident}{\it Det-Max Identifiable PMF Generative Setting}: The generative data model described by (\ref{eq:PMFgen}) and (\ref{eq:PMFgen2}) is called ``Det-Max identifiable PMF setting" if all of the  solutions of the corresponding {\it Det-Max optimization problem} in (\ref{eq:detmaxoptimization}) satisfy the forms in (\ref{eq:desired1}) and (\ref{eq:desired2}).
\end{definition}

\subsection{Proposed PMF Identifiability Condition}
\label{sec:proposedidentifiability}
One of the basic premises of this article is to provide an identifiability condition for the Det-Max optimization problem introduced in Section \ref{sec:crit}. This criterion assumes that the columns of the generative model matrix $\bS_g$ are well spread inside $\Pcal$. The Det-Max optimization problem in (\ref{eq:detmaxoptimization}) targets the dispersal of the columns of $\bS$ to exploit this assumption. The identifiability condition offered in this section is a geometric condition on the columns of $\bS_g$, to guarantee their sufficient scattering in $\Pcal$.

Earlier Det-Max optimization based BCA approaches used the inclusion of the polytope vertices as the sufficient identifiability condition for $\Ball_\infty$ \cite{erdogan2013class} and $\Ball_1$  \cite{babatas2018algorithmic}. This assumption resembles the ``separability", ``pure pixel" or ``groundedness" sufficient condition of the NMF/SSMF frameworks  \cite{donoho2003does,laurberg2008theorems}, requiring that latent vectors include the vertices of the unit simplex or their scaled versions. If we assume that the latent vectors are randomly drawn from their domains, the probability of the vertex inclusion is very low. Therefore, from the practical plausibility standpoint, we desire less stringent sufficient conditions.
\begin{figure}[!tbp]
    \centering
    \begin{minipage}[b]{0.16\textwidth}
    \includegraphics[width=0.95\textwidth,trim={6.0cm 5cm 8cm 4.3cm},clip
    ]{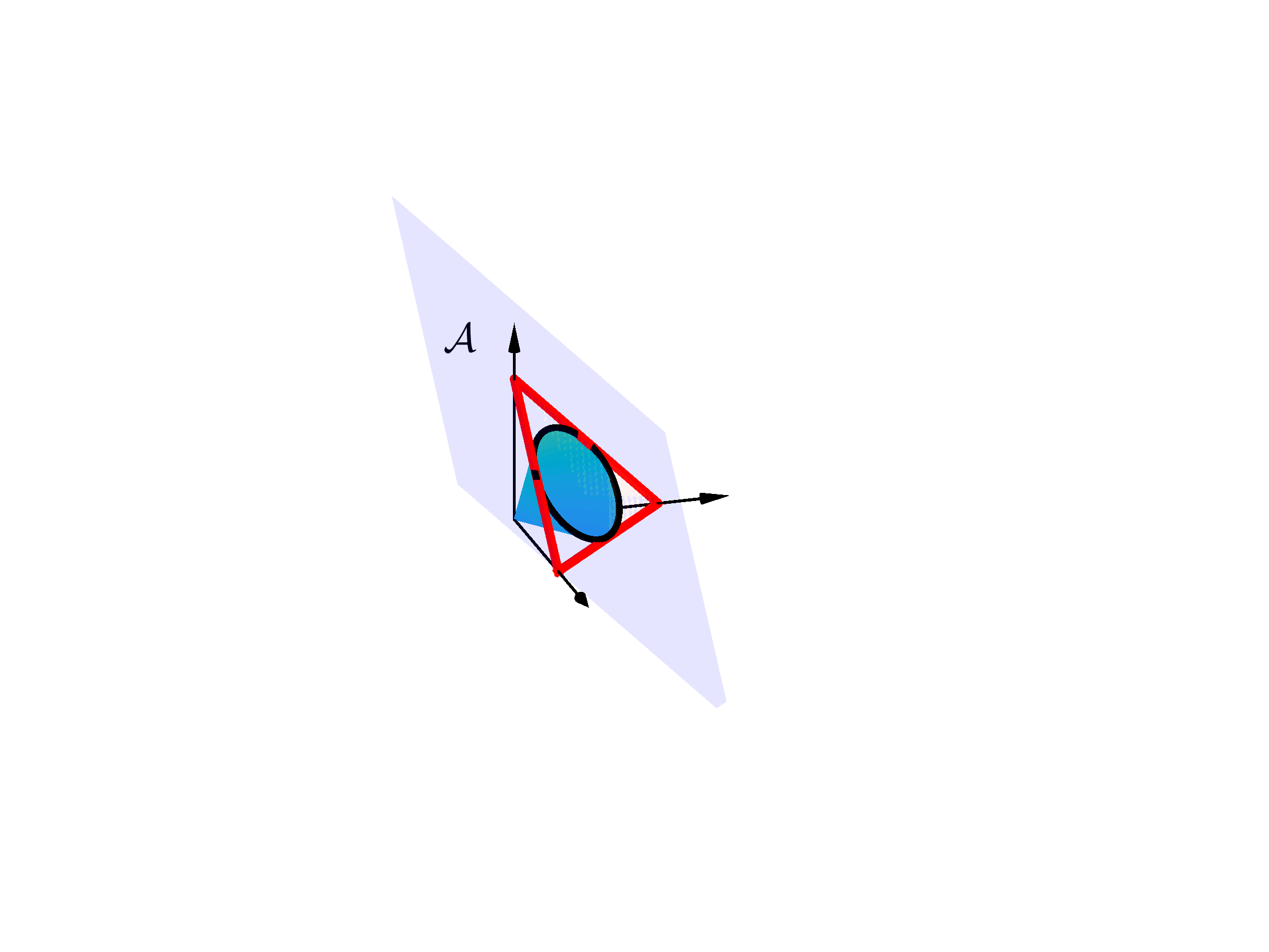}
    \caption*{(a)}
    \label{fig:nmf1}
    \end{minipage}\hfill
    \begin{minipage}[b]{0.20\textwidth}
    \includegraphics[width=0.8\textwidth
    ,trim={3.0cm 0.5cm 2cm 0.3cm},clip
    ]{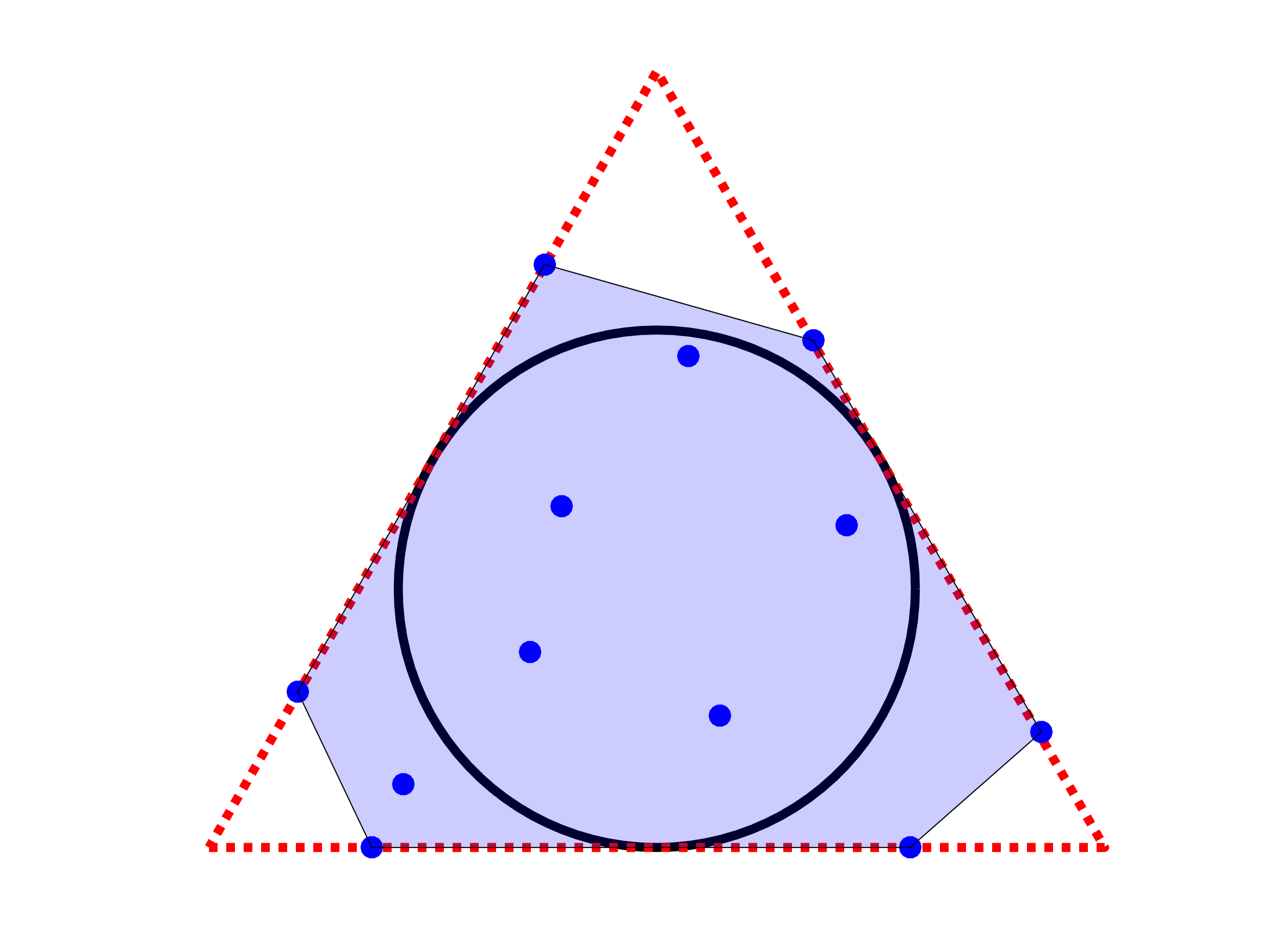}
    \caption*{(b)}
    \label{fig:nmf2}
    \end{minipage}
    \caption{Sufficient scattering condition for NMF: (a) 3D illustration of the unit simplex $\Delta_r$ (triangle with the red boundary), the second order cone $\mathcal{C}$, and the hyperplane $\mathcal{A}$. (b) an illustration of sufficiently scattered samples (dots) and their convex hull in relation to $\mathcal{C}\cap \mathcal{A}$.}
    \label{fig:nmfss}
\end{figure}
To address this issue, weaker ``sufficiently scattered conditions" were introduced for NMF. These conditions are mainly based on the enclosure of the second order cone $\mathcal{C}=\{\bx \mid \bx^T \mathbf{1}\ge \sqrt{r-1}\|\bx\|_2,\bx\in\mathbb{R}^r\}$ as a measure of spread inside the nonnegative orthant.
Figure \ref{fig:nmfss}(a) illustrates $\mathcal{C}$ for the three-dimensional case. A common approach employed in NMF algorithms is to preprocess inputs to enforce unit $\ell_1$-norm constraints on the nonnegative latent vectors. As a result of this normalization, the original vectors in nonnegative orthant $\mathbb{R}_+^r$ are mapped to the unit simplex $\Delta_r=\{\bx \mid \mathbf{1}^T\bx=1, \bx\ge0, \bx \in \mathbb{R}^r\}$, which is the region with the triangular boundary in Figure \ref{fig:nmfss}(a). Due to this mapping, we can focus on the $r-1$ dimensional affine subspace $\mathcal{A}=\{\bx \mid \mathbf{1}^T \bx=1,\bx\in\mathbb{R}^r\}$. Figure \ref{fig:nmfss}(b) illustrates the restriction to $\mathcal{A}$, where the latent vectors are represented with the dots, and $\text{bd}(\mathcal{C}\cap\mathcal{A})$ is the circle.

Based on this geometric setting, Fu \textit{et al.} \cite{fu2018identifiability} proposed the combination of the following conditions for sufficient scattering: 
\begin{itemize}
    \item[] (NMF.SS.i) $\text{cone}(\bS)\supseteq \mathcal{C}$, 
    \item[] (NMF.SS.ii)
    $\text{cone}(\bS)^d \cap \text{bd}( \mathcal{C}^d)=\{\gamma \mathbf{e}_k \mid  \gamma>0, k=1, \ldots r\}$. 
\end{itemize}
The notation $\mathcal{K}^d$ in (NMF.SS.ii) represents the dual cone of $\mathcal{K}$, as defined in Table \ref{tab:notation}.
The first condition, (NMF.SS.i), ensures that the conic hull of the columns of $\bS$ contains $\mathcal{C}$. Restricted to the affine subspace $\mathcal{A}$, this condition is equivalent to that in which the convex hull of the (normalized) latent vectors, the purple shaded region in Figure \ref{fig:nmfss}(b), contains $\mathcal{C}\cap\mathcal{A}$. The second condition, (NMF.SS.ii), limits the tightness of the enclosure by  constraining the points of tangency between $\mathcal{C}$ and $\text{cone}(\bS)$.

Lin \textit{et al.} \cite{lin2015identifiability} introduced an alternative but related  \cite{fu2016robust} condition for SSMF, which is based on the set 
$\mathcal{R}(a)=(a\Ball_2) \cap \Delta_r$,
i.e., the intersection of the origin centered hypersphere with radius $a$ and the unit simplex. The constant $\gamma=\sup\{a\le 1 \mid \mathcal{R}(a)\subseteq \text{conv}(\bS)\}$, where the columns of $\bS$ are in $\Delta_r$, is defined as the uniform pixel purity level. The sufficiently scattered condition in  \cite{lin2015identifiability} requires that $\gamma>\frac{1}{r}$.  

In this article, we extend the sufficient scattering condition approach introduced for NMF in \cite{fu2018identifiability} to PMF. For this purpose, we replace the second order cone $\mathcal{C}$ in NMF with the MVIE of the polytope. The MVIE serves as the reference object to measure the spread of the latent vectors inside $\Pcal$. The MVIE of a polytope $\mathcal{P}$ can be represented with (see Section 8.4.2, Page 400 of  \cite{boyd2004convex})
\begin{eqnarray}
 \mathcal{E}_\mathcal{P}=\{\bC_\mathcal{P}\bu+\bg_\mathcal{P} \mid \|\bu\|_2\le 1\}, \label{def:ellipsoid}
 \end{eqnarray}
where, for a polytope defined by (\ref{eq:ptope1}), the pair $(\bC_\mathcal{P}\in\mathbb{R}^{r\times r},\bg_{\mathcal{P}}\in \mathbb{R}^{r})$ is obtained as the optimal solution of the optimization problem:
\begin{mini!}[l]<b>
{\bC\in\mathbb{R}^{r\times r},\bg \in \mathbb{R}^{r}}{-\log\det\bC\label{app:mvieobjective}}{\label{app:mvie}}{}
\addConstraint{\|\bC\ba_i\|_2 +\ba_i^T\bg}{\le b_i,\quad}{i=1, \ldots, f}
\addConstraint{\bC}{\succeq\mathbf{0}.}{}
\end{mini!}
The following theorem by Fritz John  \cite{ball1992ellipsoids} is useful for identifying spherical MVIEs:
\begin{theorem}\label{johnstheorem} $\Ball_2$ is the ellipsoid of maximal volume contained in the convex body $C\subset\mathbb{R}^r$ if and only if $\Ball_2\subset C$ and, for some $m\ge r$, there are unit $2$-norm vectors $\{\bu_1, \ldots, \bu_m\}\subset\mathbb{R}^r$ on the boundary of $C$, and positive numbers $\{c_i, i=1,\ldots, m\}$ for which $\sum_{i=1}^m c_i\bu_i=\mathbf{0}$ and $\sum_{i=1}^m c_i\bu_i\bu_i^T=\bI_r$.
\end{theorem}
Based on Theorem \ref{johnstheorem}, we can show that the special polytope $\Ball_\infty$ (item (i) in Section \ref{sec:PMFsetup}.) has $\Ball_2$ as its MVIE, with the choices of $\bu_i=\be_i,i=1, \ldots, r$ and $\bu_{i}=-\be_{i-r}, i=r+1, \ldots 2r$, $c_i=1, i=1, \ldots, r$. Similarly, the MVIE of $\sqrt{r}\Ball_1$ is also $\Ball_2$, which can be justified by Theorem 1, through the choice of $\bu_i=\frac{1}{\sqrt{r}}\bq_i, i=1, \ldots, 2^r$, where vectors $\bq_i$ are all possible distinct sign vectors with $\pm 1$ entries. 

We propose the following MVIE-based sufficiently scattering condition to be used in the identifiability results.
\begin{definition}\label{suffscat}{\it Sufficiently Scattered Factor}: 
 $\bS \in \mathbb{R}^{r \times N}$ is called a sufficiently scattered factor corresponding to $\mathcal{P}$ if
\begin{itemize}
\item[]{\it (PMF.SS.i)} $\mathcal{P} \supseteq \text{conv}(\bS)\supset \mathcal{E}_\mathcal{P}$, and
\item[]{\it (PMF.SS.ii)} $\text{conv}(\bS)^{*,\bg_\mathcal{P}} \cap \text{bd}(\mathcal{E}_\mathcal{P}^{*,\bg_\mathcal{P}})=\text{ext}(\mathcal{P}^{*,\bg_\mathcal{P}})$,
\end{itemize}
where $\mathcal{E}_\mathcal{P}$ is the MVIE of $\mathcal{P}$, centered at $\bg_\mathcal{P}$.
\end{definition}
\begin{figure}[H]
\centering
\begin{subfigure}[b]{\figsize\textwidth}
\includegraphics[width=0.79\textwidth,trim={0cm 1.05cm 2.0 0.95cm},clip]{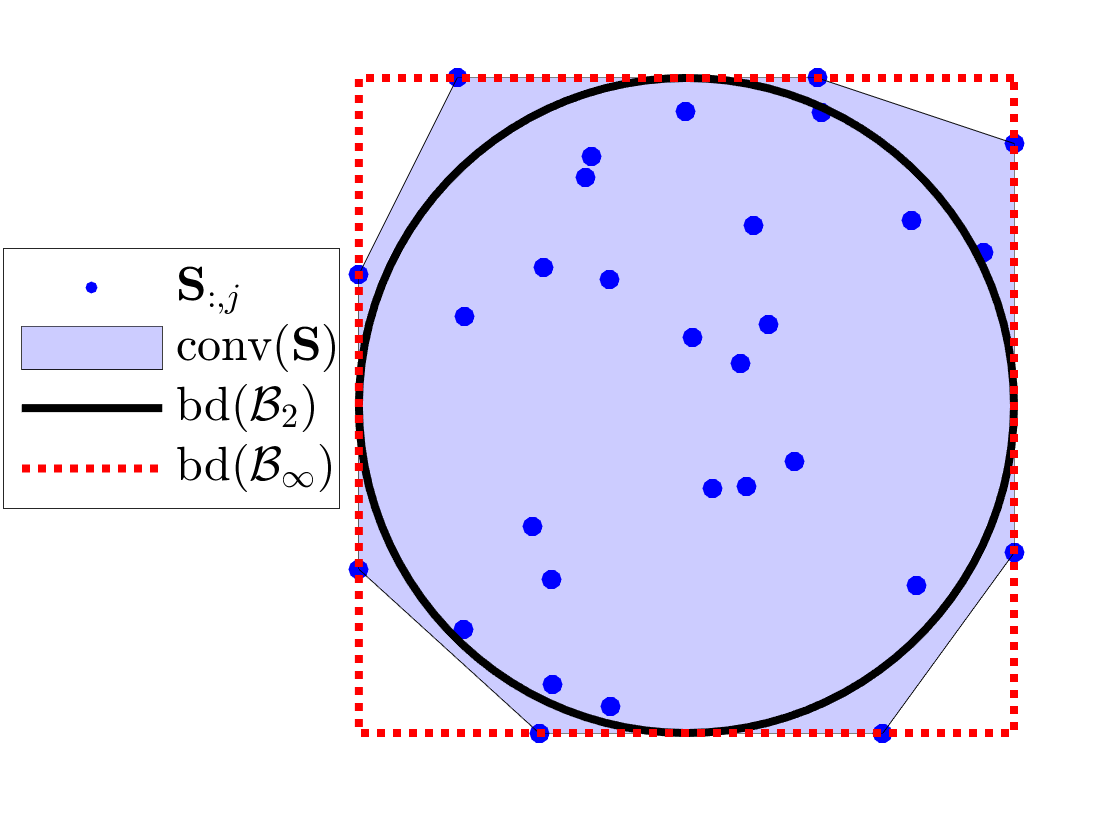}
\end{subfigure}
\caption{{\it Sufficiently scattered case} example for $2$D antisparse PMF.}
\label{fig:suffsc2D}
\end{figure}
The condition (PMF.SS.i) in Definition \ref{suffscat} guarantees that the convex hull of the columns of $\bS$ contains the MVIE of $\mathcal{P}$. Figure \ref{fig:suffsc2D} illustrates a set of sufficiently scattered samples for $\Pcal=\Ball_\infty \subset \mathbb{R}^2$. Here the square region (with red borders) is $\Ball_\infty$ polytope, the circle is the boundary of its MVIE $\Ball_2$, the dots represent the sufficiently scattered samples, and the purple shaded region is the convex hull of these samples. One can consider MVIE in Figure \ref{fig:suffsc2D} to be an ellipsoidal approximation of the polytope. The condition (PMF.SS.i) essentially ensures that the convex hull of the samples forms a better approximation of $\Pcal$ than its MVIE.
\begin{figure}[ht]
\centering
\begin{subfigure}[b]{\figsize\textwidth}
\includegraphics[width=0.78\textwidth,trim={1.5cm 1.7cm 2.0 1.3cm},clip]{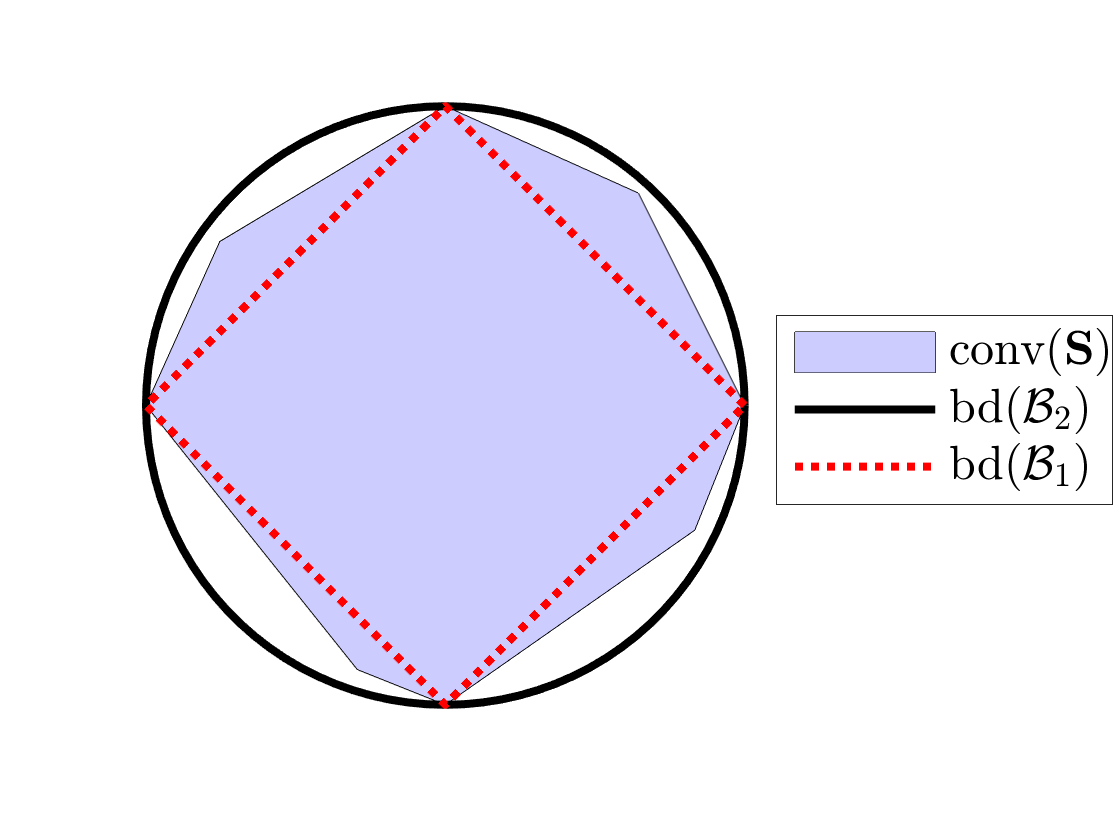}
\end{subfigure}
\caption{Polar sets for the example in Figure \ref{fig:suffsc2D}.}
\label{fig:suffsc2Dpolar}
\end{figure}
Furthermore,  the polar domain constraint (PMF.SS.ii) in Definition \ref{suffscat} places a restriction on the intersection between $\text{conv}(\bS)^{*,\bg_\mathcal{P}}$ and $\text{bd}(\mathcal{E}_\mathcal{P}^{*,\bg_\mathcal{P}})$. Figure \ref{fig:suffsc2Dpolar} provides the polar domain picture corresponding to the example in Figure \ref{fig:suffsc2D}. The polar of $\Ball_\infty$ is $\Ball_1$, the boundary of which  is shown with the dotted (red) lines. This is due to the fact that the face normals of $\Ball_{\infty}$, which are the standard basis vectors and their negatives, are the vertices of the polar polytope (see Appendix \ref{app:convex}). The polar of the MVIE $\Ball_2$ is equal to itself. Here, we can clearly observe the reversal of the inclusion relationship $\mathcal{P}\supseteq \text{conv}(\bS)\supset \mathcal{E}_\mathcal{\Pcal}$ in the sample domain as $\mathcal{E}_\Pcal^{*,\bg_\Pcal}\supset\text{conv}(\bS)^{*,\bg_\mathcal{P}}\supseteq \mathcal{P}^{*,\bg_\mathcal{P}}$ in the polar domain. Furthermore, we observe from Figure \ref{fig:suffsc2Dpolar} that the boundary of $\mathcal{E}_{\Ball_\infty}^{*}$ intersects  $\text{conv}(\bS)^{*}$ at the vertices of $\Ball_1$, i.e., the standard basis vectors and their negatives. These intersection points in the polar domain correspond to the normals of hyperplanes where $\text{conv}(\bS)$ and $\Ball_\infty$ are tangent to the boundary of the MVIE $\Ball_2$ in Figure \ref{fig:suffsc2D}.  As illustrated by this example in Figures \ref{fig:suffsc2D} and \ref{fig:suffsc2Dpolar}, the polar domain constraint  (PMF.SS.ii) limits the points of tangency between $\text{conv}(\bS)$ and  $\mathcal{E}_\mathcal{P}$ to the intersection of the polytope $\Pcal$ and the boundary of its MVIE $\text{bd}(\mathcal{E}_\mathcal{\Pcal})$.  Therefore, we can consider (PMF.SS.ii) to be a constraint on how tightly $\text{conv}(\bS)$ can enclose $\mathcal{E}_\mathcal{\Pcal}$.
\section{Special PMF Cases}
 \label{sec:specialPMF}
This section focuses on the special polytopes introduced in Section \ref{sec:PMFproblem} corresponding to the combination of multiple component attributes such as sparse/antisparse and nonnegative/signed due to their practical relevance in existing applications. We provide the generalization of all of the identifiable polytopes in Section \ref{sec:generalizedpmf}.
\subsection{Antisparse PMF}
\label{sec:antisparse}
The antisparse case corresponds to the setting in which the columns of the $\bS_g$ matrix are distributed inside the $\ell_\infty$-norm-ball; i.e., $\Pcal=\Ball_\infty$ as defined in item (i) in Section \ref{sec:PMFproblem}.
In the sufficiently scattered case,  $\bS_g$ has columns for which  near-maximum magnitude values are simultaneously achieved for all of their components, hence the name antisparse  \cite{fuchs2011spread,jegou2012anti}. Such factorization has also been referred to as ``democratic representations"  \cite{studer2014democratic}. The reference \cite{erdogan2013class} proposed a BCA framework that exploited the use of $\Ball_\infty$ as the domain of latent vectors. This approach also employs the determinant maximization criterion. However, instead of defining an optimization problem with a $\Ball_\infty$ constraint, it proposes an unconstrained optimization problem. The antisparse BCA objective function contains a penalty term corresponding to the ``size" of the minimum volume $\Ball_\infty$ polytope enclosing latent vectors. The identifiability results offered in \cite{erdogan2013class} assumed  that the latent vectors in the generative model contain all of the vertices of $\Ball_\infty$. In this section, we  show that we can replace the vertex-inclusion assumption with the less stringent  sufficiently scattering assumption in Definition \ref{suffscat} for  $\Ball_\infty$. In Section \ref{sec:proposedidentifiability}, using Theorem \ref{johnstheorem}, we showed that the MVIE for $\Ball_\infty$ is $\Ball_2$. Therefore, the corresponding MVIE parameters in description  (\ref{def:ellipsoid}) are $\bC_{\Ball_\infty}=\mathbf{I}$ and $\bg_{\Ball_\infty}=\mathbf{0}$.

\begin{figure}[!ht]
    \centering
    \begin{minipage}[b]{0.18\textwidth}
    \includegraphics[width=0.99\textwidth,trim={1.2cm 0.85cm 1.65cm 0.85cm},clip]{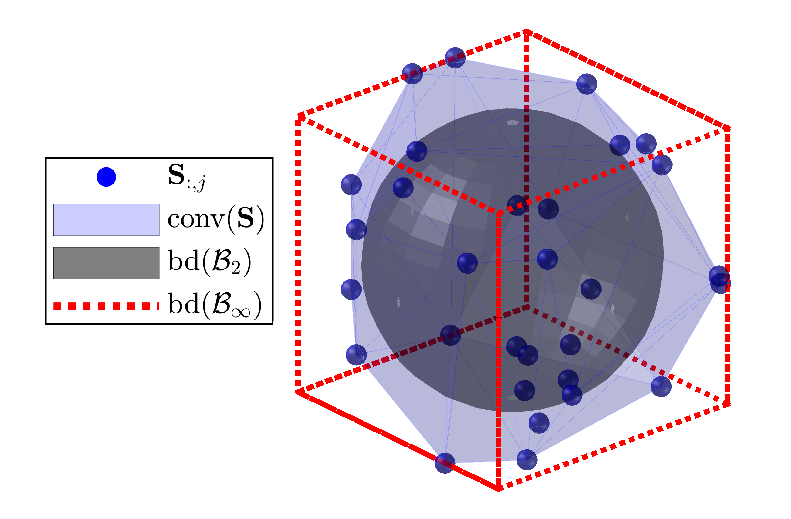}
    \caption*{(a)}
    \label{fig:as3d1}
    \end{minipage}\hfill
    \begin{minipage}[b]{0.205\textwidth}
   \includegraphics[width=0.99\textwidth,trim={1.0cm 0.9cm 1.6cm 0.95cm},clip]{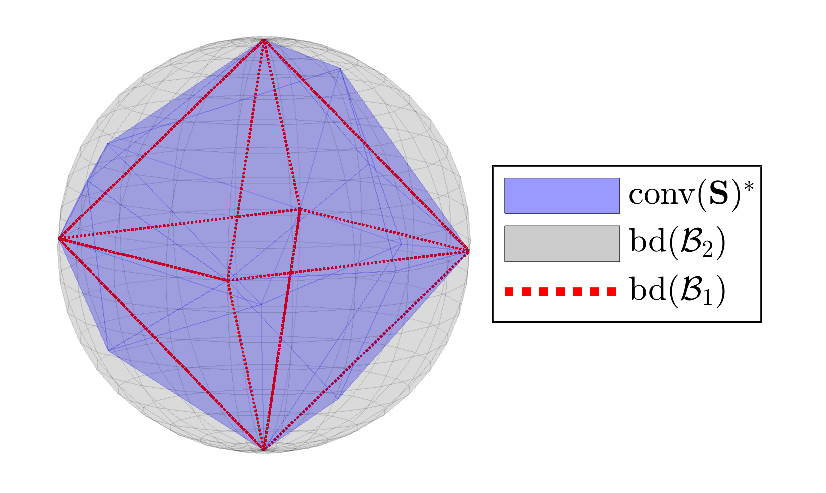}
    \caption*{(b)}
    \label{fig:a3d2}
    \end{minipage}
    \caption{Sufficiently scattered case example for $3$D antisparse PMF: (a) sample domain, (b) polar domain.}
    \label{fig:suffsc3D}
\end{figure}

In Section \ref{sec:proposedidentifiability}, Figures \ref{fig:suffsc2D} and \ref{fig:suffsc2Dpolar} offered illustrations for the convex hull of sufficiently scattered samples in $\Ball_\infty$, the polytope $\Ball_\infty$, its MVIE  $\mathcal{E}_{\Ball_\infty}$ and their polars, in a two-dimensional setting. 
Figure \ref{fig:suffsc3D}(a) illustrates a sufficiently scattered selection of the columns of $\bS$ for the three-dimensional case ($r=3$), where the vertices of $\Ball_\infty$ are not included in the samples. The corresponding three-dimensional polar domain picture is provided in Figure \ref{fig:suffsc3D}(b).

The following theorem characterizes the identifiability of the antisparse PMF problem under the proposed sufficiently scattered condition on $\bS_g$.
\begin{theorem}
\label{thm1}
Given the general PMF setting outlined in Section \ref{sec:PMFproblem}, if $\bS_g$ is a sufficiently scattered factor for {\it antisparse PMF} according to Definition \ref{suffscat}, then all global optima $\bH_*,\bS_*$ of the  Det-Max optimization problem in (\ref{eq:detmaxoptimization}) for $\Pcal=\Ball_\infty$ satisfy
\begin{eqnarray}
\bH_*&=&\bH_g\bPi^T\bD, \label{eq:thm1}\\
\bS_*&=&\bD\bPi\bS_g,  \label{eq:thm2}
\end{eqnarray}
where $\bPi\in\mathbb{R}^{r \times r}$ is a permutation matrix and $\bD\in\mathbb{R}^{r \times r}$ is an invertible diagonal matrix with $\pm 1$ entries on its diagonal.
\end{theorem}
\begin{proof}
Due to the full column rank condition on $\bH_g$ and the  constraint $\bY=\bH\bS$ in (\ref{eq:detmaxconstr1}), any feasible point $\bS$ has the same row space as $\bS_g$, which implies
\begin{eqnarray}
\bS=\bA\bS_g, \label{eq:feasibleS}
\end{eqnarray}
for some full rank $\bA\in \mathbb{R}^{r \times r}$  matrix. Therefore, finding the optimal choice of $\bS$ boils down to  finding the optimal  $\bA$.
Using the parametrization in (\ref{eq:feasibleS}), the  objective function in (\ref{eq:detmaxobjective}) of the  Det-Max optimization problem can be written as
\begin{eqnarray*}
\det(\bS\bS^T)=|\det(\bA)|^2\det(\bS_g\bS_g^T).
\end{eqnarray*}
Since the second term on the right-hand side is constant, the optimization objective function can be reduced to $|\det(\bA)|$. Therefore, we can write the equivalent problem of (\ref{eq:detmaxoptimization}) for the antisparse PMF case as
\begin{maxi!}[l]<b>
{\bA\in \mathbb{R}^{r \times r}}{|\det(\bA)|\label{eq:antisparseeqobjective}}{\label{antisparseeqproblem}}{}
\addConstraint{\|\bA{\bS_g}_{:,j}\|_\infty}{\le 1,\quad}{j=1, \ldots, N. \label{eq:eqantisparseconstr}}
\end{maxi!}
The remaining proof consists of three main steps.

\uline{The first step:}\textit{ Using the constraint in (\ref{eq:eqantisparseconstr}), we first show that the rows of any feasible $\bA$ are in $\text{conv}(\bS_g)^*$}.

For this purpose, using (\ref{eq:eqantisparseconstr}), we can write  $\bA_{i,:}{\bS_g}_{:,j}\le 1$,
for all $(i,j)$ index pairs, which  further  implies that $\bA_{i,:}$ satisfies
\begin{eqnarray}
\bA_{i,:}{\bs} =\bA_{i,:}\bS_g\blambda\le 1, \hspace{0.1in} \blambda\in \mathbb{R}_+^N, \hspace{0.02in}\mathbf{1}^T\blambda=1, \nonumber
\end{eqnarray}
for all $\bs\in \text{conv}(\bS_g)$ and $i \in \{1, \ldots, r\}$. This condition is equivalent to each row of $\bA$ lying in the polar of $\text{conv}(\bS_g)$, i.e.,
\begin{eqnarray*}
 \bA^T_{i,:}\in \text{conv}(\bS_g)^*, \hspace{0.1in} \forall i \in \{1, \ldots, r\}.
 \end{eqnarray*}
 In reference to the three-dimensional polar domain picture in Figure \ref{fig:suffsc3D}(b), the rows of $\bA$ lie in the polytopic shaded region corresponding to $\text{conv}(\bS_g)^*$.
 
 \uline{The second step:} \textit{ Using (PMF.SS.i) and Hadamard's inequality, we show that any optimal solution $\bA_*$ of (\ref{antisparseeqproblem}) should be a real orthogonal matrix.}
 
The polar version of the sufficient scattering condition (PMF.SS.i), i.e.,  $\text{conv}(\bS_g)\supset \mathcal{E}_\Pcal$, is equivalent to $\text{conv}(\bS_g)^*\subset \Ball_2$. Therefore, we have $\|\bA_{i,:}\|_2\le 1$ for all rows of $\bA$.   This finding  implies that the rows of $\bA$ lie inside the unit sphere in Figure \ref{fig:suffsc3D}(b).
Hadamard's inequality-based bound
on the objective function in (\ref{eq:antisparseeqobjective}),
\begin{eqnarray*}
|\det(\bA)|\le  \|\bA^T_{1,:}\|_2 \|\bA^T_{2,:}\|_2 \ldots \|\bA^T_{r,:}\|_2\le 1, 
 \end{eqnarray*}
is achieved if and only if the rows of $\bA$ are on $\text{bd}(\Ball_2)$ and they are mutually orthogonal. In other words, any optimal solution $\bA_*$ is a real orthogonal matrix. Therefore, for the example case in Figure \ref{fig:suffsc3D}, the rows of $\bA_*$  should lie on the boundary of the unit sphere in Figure \ref{fig:suffsc3D}(b). At the same time,  they  should be members of $\text{conv}(\bS_g)^*$, the purple shaded region in the same figure. 

\uline{The third step:} \textit{ Using  the polar domain constraint (PMF.SS.ii), we show that any optimal real orthogonal $\bA_*$ has only one nonzero element at each row (column)}.

The sufficiently scattered condition (PMF.SS.ii), $\text{conv}(\bS_g)^*\cap \text{bd}(\Ball_2)=\text{ext}(\Ball_1)$, 
restricts the rows of any optimal solution $\bA_*$, which are located on the boundary of the unit sphere in Figure \ref{fig:suffsc3D}(b), to the vertices of $\Ball_1$. This condition  
implies that  the rows of $\bA_*$ are standard basis vectors (or their negatives). Therefore, we can write
 $\bA_*=\bD\bPi$,
where $\bD\in\mathbb{R}^{r \times r}$ is a diagonal matrix with $\pm 1$ entries on its diagonal and $\bPi\in\mathbb{R}^{r \times r}$ is a permutation matrix. Due to the equality constraint in  (\ref{eq:detmaxconstr1}), $\bH_*=\bH_g\bA^{-1}$; therefore,  (\ref{eq:thm1}) and (\ref{eq:thm2})  follow.\end{proof}

\subsection{Sparse PMF}
\label{sec:sparse}
In the sparse PMF setting defined in item (ii) in Section \ref{sec:PMFproblem}, the columns of $\bS_g$ are located inside the $\ell_1$-norm-ball; i.e., $\Pcal=\Ball_1$. The connection between sparsity and  $\ell_1$-norm constraints  was well established in   \cite{donoho2006most,candes2008introduction}. It was shown that, under some practically plausible assumptions, the $\ell_1$-norm  acts as a convex surrogate for the $\ell_0$-norm which counts the number of non-zero elements in a given vector.
The sparsity property has also been exploited in different unsupervised  approaches (see, for example, \cite{georgiev2005sparse, elad2010sparse,comon2010handbook} and references therein). In particular,  \cite{babatas2018algorithmic} adopted the determinant maximization based  antisparse BCA approach in  \cite{erdogan2013class} to the sparse case by replacing the minimum volume  enclosing $\Ball_\infty$ with its $\Ball_1$ counterpart. The identifiability result for the sparse BCA  in \cite{babatas2018algorithmic} is based on the assumption that $\bS_g$ contains all vertices of $\Ball_1$. In this section, we show that we can relax this condition  using the sufficiently scattered condition in Definition \ref{suffscat}.

In Section \ref{sec:proposedidentifiability}, using Theorem \ref{johnstheorem}, we  show that the MVIE of $\sqrt{r}\Ball_1$ is $\Ball_2$. Therefore, the MVIE of $\Ball_1$ is $\mathcal{E}_{\Ball_1}=\frac{1}{\sqrt{r}}\Ball_2$. The polar of a hypersphere is another hypersphere with the reciprocal radius. Therefore, $\mathcal{E}_{\Ball_1}^*=\sqrt{r}\Ball_2$.

\begin{figure}[!ht]
    \centering
    \begin{minipage}[b]{0.22\textwidth}
    \includegraphics[width=0.99\textwidth,trim={1.2cm 0.85cm 1.65cm 0.8cm},clip]{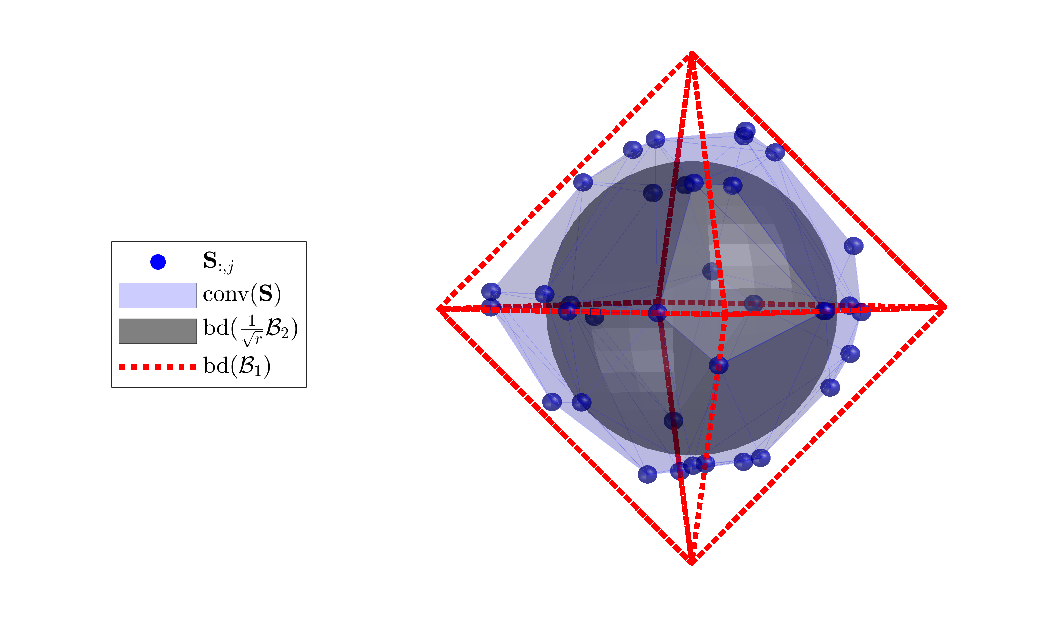}
    \caption*{(a)}
    \label{fig:s3d1}
    \end{minipage}\hfill
    \begin{minipage}[b]{0.225\textwidth}
   \includegraphics[width=0.99\textwidth,trim={1.0cm 0.9cm 1.2cm 0.95cm},clip]{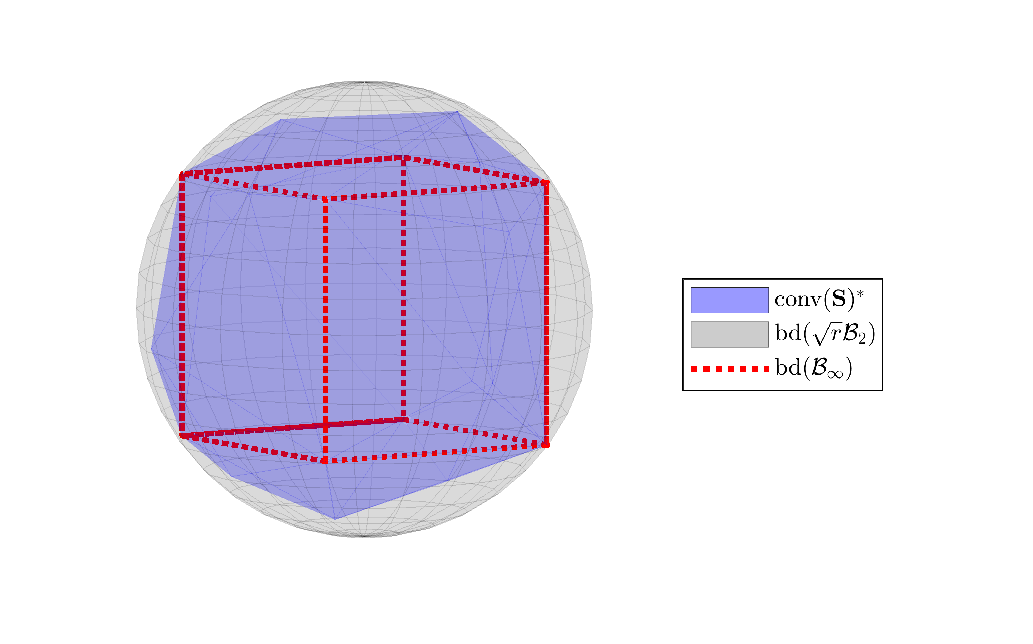}
    \caption*{(b)}
    \label{fig:s3d2}
    \end{minipage}
    \caption{Sufficiently scattered case example for $3$D sparse PMF: (a) sample domain, (b) polar domain.}
    \label{fig:sparsesuffsc3D}
\end{figure}

For a visual illustration of the sufficient scattering condition for $\Pcal=\Ball_1$, we consider the example in Figure \ref{fig:sparsesuffsc3D}(a)  for $r=3$. The sample points, represented by dots, in Figure \ref{fig:sparsesuffsc3D}(a) do not contain the vertices of $\Ball_1$. Furthermore, both $\text{bd}(\Ball_1)$ and the edges of $\text{conv}(\bS)$ intersect $\text{bd}(\frac{1}{\sqrt{r}}\Ball_2)$ at identical points due to the polar domain sufficient scattering constraint (PMF.SS.ii), as illustrated in Figure \ref{fig:sparsesuffsc3D}(b). The polar of $\Ball_1$ is $\Ball_\infty$, the boundary of which is plotted as the red cube in Figure \ref{fig:sparsesuffsc3D}(b). Its vertices and $\text{conv}(\bS)^*$ intersect the boundary of $\mathcal{E}_\Pcal^*$ at identical points.

 The following theorem characterizes the  identifiability of the sparse PMF problem, based on the sufficiently scattering condition in Definition \ref{suffscat}:
\begin{theorem}
Given the general PMF setting outlined in Section \ref{sec:PMFproblem}, if $\bS_g$ is a sufficiently scattered factor for {\it sparse PMF} according to Definition \ref{suffscat}, then all global optima $\bH_*,\bS_*$ of the Det-Max optimization problem in (\ref{eq:detmaxoptimization}) with $\Pcal=\Ball_1$ satisfy
\begin{eqnarray*}
\bH_*&=&\bH_g\bPi^T\bD,\\
\bS_*&=&\bD\bPi\bS_g,
\end{eqnarray*}
where $\bPi\in\mathbb{R}^{r \times r}$ is a permutation matrix and $\bD\in\mathbb{R}^{r \times r}$ is a diagonal matrix with $\pm 1$ entries on its diagonal.
\end{theorem}
\begin{proof}
Using the same arguments in the proof of Theorem \ref{thm1}, we write the optimization problem equivalent to (\ref{eq:detmaxoptimization}) for the sparse PMF case as
\begin{maxi!}[l]<b>
{\bA\in \mathbb{R}^{r \times r}}{|\det(\bA)|\label{eq:sparseeqobjective}}{\label{sparseeqproblem}}{}
\addConstraint{\|\bA{\bS_g}_{\:,j}\|_1} {\le 1,\quad}{j=1, \ldots, N. \label{eq:eqsparseconstr}}
\end{maxi!}
The proof consists of three fundamental steps:

\uline{The first step:}\textit{ We show   that (\ref{eq:eqsparseconstr}) and (PMF.SS.i) imply any feasible $\bA^T$ maps $\Ball_\infty$  into  $\sqrt{r}\Ball_2$.}

We start by noting that  $\|\bA\bx\|_1=\langle \bx,\bA^T\text{sign}(\bA\bx)\rangle$, for any $\bx \in \mathbb{R}^r$. Therefore,  from  (\ref{eq:eqsparseconstr}), we can write
\begin{eqnarray}
 \langle {\bS_g}_{:,j},\bA^T\bq_i\rangle \le 1,\hspace{0.1in} j=1, \ldots, N, \hspace{0.03in} \forall\bq_i\in\text{ext}(\Ball_\infty), \label{eq:sparsepolarineq}
\end{eqnarray}
where we used  $\text{ext}(\Ball_\infty)$ as the set of all possible sign vectors.
From (\ref{eq:sparsepolarineq}), we conclude that
\begin{eqnarray}
\bA^T\bq_i\in \text{conv}(\bS_g)^*, \hspace{0.2in} \forall\bq_i\in\text{ext}(\Ball_\infty). \label{eqcond}
\end{eqnarray}
In other words, any feasible $\bA^T$ maps the vertices of $\Ball_\infty$ into  $\text{conv}(\bS_g)^*$, purple shaded polytopic region in Figure \ref{fig:sparsesuffsc3D}(b).
Since $\bS_g$ is a sufficiently scattered factor, we have  $\text{conv}(\bS_g)^*\subset \sqrt{r}\Ball_2$  by (PMF.SS.i). Therefore, 
\begin{eqnarray}
\bA^T\bq_i\in \sqrt{r}\Ball_2,\hspace{0.2in}
\forall\bq_i\in\text{ext}(\Ball_\infty). \label{eqcond2}
\end{eqnarray}
(\ref{eqcond2}) further implies
$\bA(\Ball_\infty)\subset\sqrt{r}\Ball_2$,  using  the convexity of $\Ball_2$ and that $\Ball_\infty$ is the convex hull of $\text{ext}(\Ball_\infty)$. Therefore, the image of $\Ball_\infty$ under $\bA^T$ lies inside the spherical region in Figure \ref{fig:sparsesuffsc3D}(b).

\uline{The second step:}\textit{ We show that any $\bA$ satisfying (\ref{eqcond2}) and maximizing $|\det(\bA)|$ is a real orthogonal matrix}

Replacing $\text{conv}(\bS_g)^*$ in (\ref{eqcond}) with the larger set $\sqrt{r}\Ball_2$, we obtain the following optimization problem, the solution of which provides an upper bound for (\ref{sparseeqproblem}):
\begin{maxi!}[l]<b>
{\bA\in \mathbb{R}^{r \times r}}{|\det(\bA)|\label{eq:sparseubobjective}}{\label{sparseub1problem}}{}
\addConstraint{\|\bA^T\bq_i\|_2^2} {\le r,\quad}{\forall\bq_i\in\text{ext}(\Ball_\infty).\label{sparseub1constr}}
\end{maxi!}
Further relaxing  (\ref{sparseub1constr}) by totaling its  individual constraints, we obtain an alternative optimization for obtaining another upper bound:
\begin{maxi!}[l]<b>
{\bA\in \mathbb{R}^{r \times r}}{|\det(\bA)|\label{eq:sparseub2objective}}{\label{sparseub2problem}}{}
\addConstraint{\sum_{\bq_i \in \text{ext}(\Ball_\infty)}\|\bA^T\bq_i\|_2^2} {\le  2^r r.\quad \label{eq:ub2constr}}{}
\end{maxi!}
We now show that globally optimal solutions for (\ref{sparseub2problem}) are real orthogonal matrices,  which are also feasible, and therefore optimal, for the  problem in (\ref{sparseub1problem}).
First, we note that the constraint (\ref{eq:ub2constr}) can be rewritten more compactly as
$Tr(\bA\bA^T\bQ\bQ^T)\le 2^r r$, where $\bQ=\left[\begin{array}{cccc} \bq_1 & \bq_2 & \ldots & \bq_{2^r} \end{array}\right] \in \mathbb{R}^{r \times 2^r}$. Due to the symmetry of the extreme points of $\text{ext}(\Ball_\infty)$, we have $\bQ\bQ^T=2^r\bI$.
Therefore, the constraint (\ref{eq:ub2constr}) is further simplified to
$\|\bA\|_F^2\le  r$,
which can be written in terms of  the singular values of $\bA$, \{$\sigma_i(\bA), i\in\{1,\ldots, r\}\}$, as
\begin{eqnarray}
\sqrt{\sum_{i=1}^r\sigma_i^2(\bA)}\le \sqrt{r}. \label{constrineq}
\end{eqnarray}
For the objective function in (\ref{eq:sparseub2objective}), we can write
\begin{eqnarray}
|\det(\bA)|&=& \prod_{i=1}^r \sigma_i(\bA),  \nonumber\\
 &\le& \left(\frac{1}{r}\sum_{i=1}^r \sigma_i(\bA)\right)^r, \label{ProofS1} \\
 &\le& \left(\frac{1}{r}\sqrt{\sum_{i=1}^r\sigma_i^2(\bA)}\sqrt{r}\right)^r, \label{ProofS2}\\
 &\le& 1, \label{ProofS3}
\end{eqnarray}
where (\ref{ProofS1}) is due to the arithmetic-geometric mean inequality, (\ref{ProofS2}) is due to the Cauchy-Schwarz inequality, and (\ref{ProofS3}) is due to (\ref{constrineq}). The equality holds if and only if $\sigma_1(\bA)=\sigma_2(\bA)= \ldots=\sigma_r(\bA)=1$, which is equivalent to the condition that $\bA$ is real orthogonal. We note that the real orthogonal matrix  $\bA=\mathbf{I}$ is a feasible point for (\ref{sparseeqproblem}). Therefore, the upper bound by (\ref{sparseub1problem}) and (\ref{sparseub2problem}), is achievable by the optimization in (\ref{sparseeqproblem}). Thus,   $\bA_*$ is  optimal solution of (\ref{eq:sparseeqobjective}) only if it is real orthogonal. What remains to be shown is that all of the global optima of (\ref{sparseeqproblem}) are real orthogonal matrices with the desired form.

\uline{The third step:} \textit{ We use the sufficient scattering condition (PMF.SS.ii) together with the orthogonality of optimal $\bA_*$ to show that any optimal point of (\ref{sparseeqproblem}) has the desired form, i.e., the product of a permutation and a diagonal matrix.}

Given $\bq_i \in \text{ext}(\Ball_\infty)$, we have  $\|\bq_i\|_2=\sqrt{r}$. Since optimal $\bA_*$ is a real orthogonal matrix, we have $\bA_*^T\bq_i\in \text{bd}(\sqrt{r}\Ball_2)$, $\forall \bq_i \in \text{ext}(\Ball_\infty)$.
Combining it with (\ref{eqcond}), we obtain $\bA_*^T\bq_i\in \text{conv}(\bS_g)^*\cap \text{bd}(\sqrt{r}\Ball_2)$, $\forall \bq_i \in \text{ext}(\Ball_\infty)$.
Since the assumption (PMF.SS.ii) restricts  $\text{conv}(\bS_g)^* \cap \text{bd}(\sqrt{r}\Ball_2)$  to $\text{ext}(\Ball_\infty)$, the equivalent condition for the global optimality of $\bA_*$ for (\ref{sparseeqproblem}) can be written as
\begin{subequations}
\label{eq:sparseglobalopt}
\begin{align}
& \bA_*^T\bA_*=\bI, \label{eq:sparseglobalopt1} \\
& \bA_*^T\bq_i \in \text{ext}(\Ball_\infty) \hspace{0.1in} \forall \bq_i\in\text{ext}(\Ball_\infty). \label{eq:sparseglobalopt2}
\end{align}
\end{subequations}
In other words, $\bA_*$ is a global optimum if and only if it is real orthogonal and its transpose maps the vertices of $\Ball_\infty$ in Figure \ref{fig:sparsesuffsc3D}(b) to itself.
Using (\ref{eq:sparseglobalopt}), we conclude that $\bA_*$ has only one nonzero entry in each column(row) as follows:
\begin{itemize}
    \item[i.] (\ref{eq:sparseglobalopt1}) implies $\|{\bA_*}_{:,i}\|_2=1$,
    \item[ii.] (\ref{eq:sparseglobalopt2}) implies $\|{\bA_*}_{:,i}\|_1=1$, since  for any column of $\bA_*$, $\text{sign}({\bA_*}_{:,i})\in \text{ext}(\Ball_\infty)$. Therefore,  $\bA^T_*\text{sign}({\bA_*}_{:,i})=\bq_j$,  for some $\bq_j\in\text{ext}(\Ball_\infty)$, due to (\ref{eq:sparseglobalopt2}). The $i^{\textrm{th}}$ row of $\bq_j$ is $({\bA_*}_{:,i})^T\text{sign}({\bA_*}_{:,i})=\|{\bA_*}_{:,i}\|_1$. Since all components of $\bq_j$  have magnitude $1$,
 we have $\|{\bA_*}_{:,i}\|_1=1$.
\end{itemize}
  The statements (i) and (ii) above are true if and only if ${\bA_*}_{:,i}$ has only one non-zero element with unit magnitude.
  Since the rows of $\bA_*$ are orthonormal, the global optima characterization is given by $\bA_*=\bD\bPi$, where $\bD$ and $\bPi$ are as stated in the theorem.\end{proof}

\subsection{Antisparse Nonnegative PMF}
As defined in item (iii) in Section \ref{sec:PMFproblem}, this special  case refers to the polytope choice
\begin{eqnarray*}
\Pcal=\Ball_{\infty,+}&=& \{\bx \mid \mathbf{0}\le \bx \le \mathbf{1},\mathbf{x}\in\mathbb{R}_+^r\}, \label{ineq:anti-sparsenonnegpoly}\\
&=&\Ball_\infty \cap \mathbb{R}^r_+ 
 = 0.5\Ball_\infty+0.5\mathbf{1},
\end{eqnarray*}
i.e., in essence, a scaled and translated version of $\Ball_\infty$. We apply the same affine transformation to $\mathcal{E}_{\Ball_\infty}$ to obtain the MVIE of $\Ball_{\infty,+}$: $\mathcal{E}_{\Ball_{\infty,+}}=0.5 \mathcal{E}_{\Ball_\infty}+0.5\mathbf{1}$. Therefore, the parameters of the MVIE corresponding to this polytope are given by $\bC_{\Ball_{\infty,+}}=0.5\mathbf{I}$ and $\bg_{\Ball_{\infty,+}}=0.5\mathbf{1}$.

This case is a special case of antisparse BCA covered in  \cite{erdogan2013class}, in which the existing identifiability condition is based on the vertex inclusion assumption. In this section, we provide the characterization of the identifiability condition for nonnegative antisparse PMF based on the weaker sufficient scattering assumption through the following theorem.

\begin{theorem}
Given the general PMF setting outlined in Section \ref{sec:PMFproblem}, if $\bS_g$ is a sufficiently scattered factor for {\it antisparse nonnegative PMF}  according to Definition \ref{suffscat}, then all global optima $\bH_*,\bS_*$ of the  Det-Max optimization problem in (\ref{eq:detmaxoptimization}) with $\Pcal=B_{\infty,+}$ satisfy
\begin{eqnarray*}
\bH_*&=&\bH_g\bPi^T,\\
\bS_*&=&\bPi\bS_g, 
\end{eqnarray*}
where $\bPi\in\mathbb{R}^{r \times r}$ is a permutation matrix.
\end{theorem}

\begin{proof}
Following the same treatment in the proof of Theorem \ref{thm1}, we can write the  equivalent form of the optimization in (\ref{eq:detmaxoptimization}) for nonnegative antisparse PMF as 
\begin{maxi!}[l]<b>
{\bA\in \mathbb{R}^{r \times r}}{|\det(\bA)|\label{eq:nonantisparseeqobjective}}{\label{nonantisparseeqproblem}}{}
\addConstraint{0 \le\bA_{i,:}{\bS_g}_{:,j}} {\le 1,\quad}{i=1, \ldots, r \label{eq:eqnonantisparseconstr}}
\addConstraint{}{}{j=1,\ldots, N. \nonumber}
\end{maxi!}
The proof consists of three major steps.

\uline{The first step:} \textit{ We use (\ref{eq:eqnonantisparseconstr}) and (PMF.SS.i) to show that the columns(rows) of any feasible $\bA$ of (\ref{nonantisparseeqproblem}) lie in $\Ball_2$.}

Due to (PMF.SS.i), which is  $\mathcal{E}_{\Ball_{\infty,+}}\subset\text{conv}(\bS_g)$, for all $\bs \in \mathcal{E}_{\Ball_{\infty,+}}$ the constraint (\ref{eq:eqnonantisparseconstr}) holds. Therefore, for any $i \in\{1, \ldots, r\}$, we have
\begin{eqnarray}
 0 \le0.5\bA_{i,:}{\bu}+0.5\bA_{i,:}\mathbf{1}\le 1, \hspace{0.1in} \forall \bu \in \Ball_2, \label{ineq:antisparsenonnegpoly}
 \end{eqnarray}
where we used $\mathcal{E}_{\Ball_{\infty,+}}=\{0.5\bu+0.5\mathbf{1}, \|\bu\|_2\le 1\}$.
If we substitute  $\bu=\frac{(\bA_{i,:})^T}{\|\bA_{i,:}\|_2}$ and $\bu=-\frac{(\bA_{i,:})^T}{\|\bA_{i,:}\|_2}$ in  (\ref{ineq:antisparsenonnegpoly}), we obtain
\begin{eqnarray}
0 \ \ \le&0.5\|\bA_{i,:}\|_2+0.5\bA_{i,:}\mathbf{1}&\le \ \ 1, \text{ and,} \label{ineq:antisparsenonneginc1}\\
0 \ \ \le&-0.5\|\bA_{i,:}\|_2+0.5\bA_{i,:}\mathbf{1}&\le \ \ 1, \label{ineq:antisparsenonneginc2}
 \end{eqnarray}
respectively, for all $i=1, \ldots, r$. The summation of (\ref{ineq:antisparsenonneginc1}) and the sign reversed (\ref{ineq:antisparsenonneginc2}) lead to  $\|\bA_{i,:}\|_2\le 1$ for all rows of $\bA$. In other words, the rows(columns) of any feasible $\bA$ should be in $\Ball_2$.

\uline{The second step:} \textit{ We now  show any optimal solution $\bA_*$ of (\ref{nonantisparseeqproblem}) should be a real orthogonal matrix.}

Using Hadamard's inequality and $\|\bA_{i,:}\|_2\le 1$ from the previous step,  we can place an upper bound on the objective function in (\ref{eq:nonantisparseeqobjective}) as 
$|\det(\bA)|\le  \|\bA_{1,:}\|_2 \|\bA_{2,:}\|_2 \ldots \|\bA_{r,:}\|_2\le 1$. This bound is achieved if and only if the rows form an orthonormal set. Therefore, any  optimal solution $\bA_*$ of (\ref{nonantisparseeqproblem})   is a real orthogonal matrix. 

\uline{The third step:} \textit{ We use (PMF.SS.ii) to show that all global optima of (\ref{nonantisparseeqproblem}) are  permutation matrices.} 

If we substitute a real orthogonal $\bA_*$ in (\ref{ineq:antisparsenonneginc2}), we obtain ${\bA_*}_{i,:}\mathbf{1}\ge 1$. Combining this inequality with (\ref{eq:eqnonantisparseconstr}), we can write
$2{\bA_*}_{i,:}{\bS_g}_{:,j}\le2\le 1+{\bA_*}_{i,:}\mathbf{1}$.
Reorganizing the left and right terms of this expression, we obtain
$2{\bA_*}_{i,:}({\bS_g}_{:,j}-0.5\mathbf{1})\le 1$
for all  the columns of $\bS_g$. Based on this inequality, we conclude that $2{\bA_*}_{i,:}$ lies in $\text{conv}(\bS_g)^{*,0.5\mathbf{1}}$ for all the rows of $\bA_*$. Furthermore, since $\bA_*$ is also real orthogonal,
$2{\bA_*}_{i,:}\in \text{conv}(\bS_g)^{*,0.5\mathbf{1}}\cap \text{bd}(2\Ball_2)$. 
Due to the sufficiently scattered constraint (PMF.SS.ii), which is  $\text{conv}(\bS_g)^{*,0.5\mathbf{1}}\cap \text{bd}(2\Ball_2)=\text{ext}(2\Ball_1)$, ${\bA_*}_{i,:}\in \text{ext}(\Ball_1)$. Furthermore, due to the condition ${\bA_*}_{i,:}\mathbf{1}\ge 1$, $\bA_*$ is a global optimum if and only if its rows are positive standard basis vectors that are orthogonal to each other, i.e., $\bA_*$ is a permutation matrix.
\end{proof}

\subsection{Sparse Nonnegative PMF}
As defined in item (iv) in Section \ref{sec:PMFproblem}, the polytope for  sparse  nonnegative factors is given by
\begin{eqnarray}
\Pcal=\Ball_{1,+}= \{\bx \mid \bx \ge \mathbf{0}, \mathbf{1}^T\bx \le 1 \}=\Ball_1 \cap \mathbb{R}^r_+. \label{ineq:sparsenonnegpoly}
\end{eqnarray}
There are various matrix factorization algorithms combining nonnegativity and sparsity. However, we are not aware of any approaches that make explicit mention of the polytope $\Ball_{1,+}$, and provide the corresponding identifiability conditions. In this section, we  provide identifiability results for this polytope, again based on the sufficient scattering assumption in Definition \ref{suffscat}. For this purpose, we derive the MVIE $\mathcal{E}_{\Ball_{1,+}}$ in Appendix \ref{app:mviesparsenonnegative}. The MVIE derivation in this case is relatively more involved compared to other special cases (in items (i)-(iii) in Section \ref{sec:PMFproblem}). This is due to the fact the MVIE in this case is not spherical, and therefore,  Theorem \ref{johnstheorem} is not applicable.  The inscribed ellipsoid parameters obtained in Appendix \ref{app:mviesparsenonnegative} are
\begin{eqnarray}
\bC_{\Ball_{1,+}}&=&\frac{1}{\sqrt{r}}\left( \frac{1}{\sqrt{r+1}}\mathbf{I}-\frac{\sqrt{r+1}-1}{r^2+r}\mathbf{1}\mathbf{1^T}\right), \label{eq:mvienonsparsecase} \\ \bg_{\Ball_{1,+}}&=&\frac{1}{r+1}\mathbf{1}. \nonumber
\end{eqnarray}

To obtain the polar of the polytope $\Pcal$, we first use the canonical form for the polytope shifted by $-\bg_{\Ball_{1,+}}$ in the form $\Pcal-\bg_{\Ball_{1,+}}=\{\bx\mid \mathbf{1}^T(\bx+\frac{1}{r+1}\mathbf{1})\le 1, -\be_i^T(\bx+\frac{1}{r+1}\mathbf{1}) \le 0, i\in\{1,\ldots,r\}\}$. After some algebraic manipulations, we can convert it into standard form (\ref{appApoly}) in Appendix \ref{app:convex}: $\Pcal-\bg_{\Ball_{1,+}}=\{\bx \mid (r+1)\mathbf{1}^T\bx\le 1,-(r+1)\be_i^T\bx \le 1, i\in\{1,\ldots,r\}\}$. Therefore, based on the procedure in Appendix \ref{app:convex}, the polar of $\Ball_{1,+}$ can be written as
\begin{eqnarray}
\Ball_{1,+}^{*,\frac{1}{r+1}\mathbf{1}}=\text{conv}(-(r+1)\be_1, \ldots,-(r+1)\be_r,(r+1)\mathbf{1}),\label{eq:B1pluspolar}
\end{eqnarray}
which is a polytope with $(r+1)$ vertices (Note that $\Ball_{1,+}$ has $(r+1)$ faces).
\begin{figure}[!ht]
    \centering
    \begin{minipage}[b]{0.22\textwidth}
    \includegraphics[width=\textwidth,trim={0.0cm 1.3cm 0 1.35cm},clip]{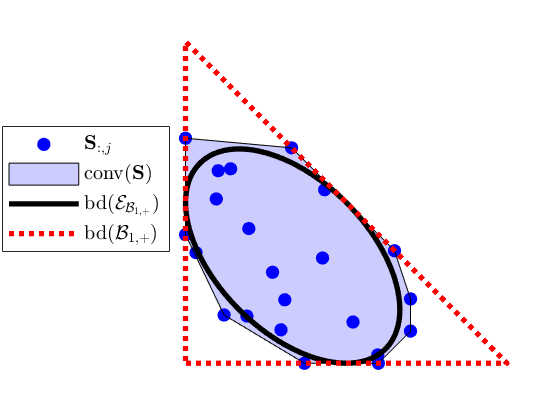}
    \caption*{(a)}
    \label{fig:s3d1}
    \end{minipage}\hfill
    \begin{minipage}[b]{0.225\textwidth}
   \includegraphics[width=\textwidth,trim={2.7cm 11.3cm 4cm 11.5cm},clip]{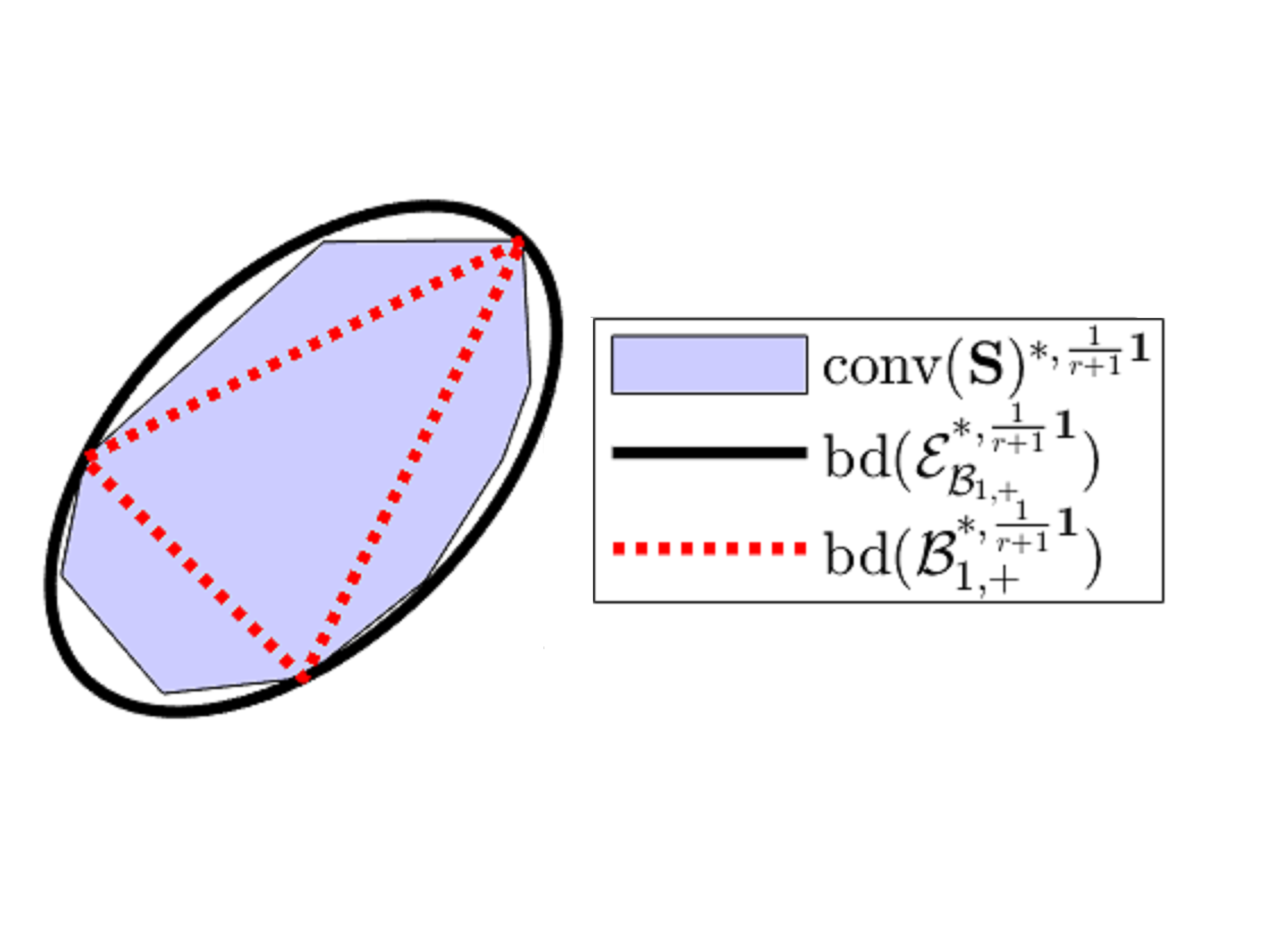}
    \caption*{(b)}
    \label{fig:s3d2}
    \end{minipage}
    \caption{Sufficiently scattered case example for $2$D sparse nonnegative PMF: (a) sample domain, (b) polar domain.}
    \label{fig:sparsenonnegsuffsc2D}
\end{figure}
Figure \ref{fig:sparsenonnegsuffsc2D}(a) illustrates $\Ball_{1,+}$, its MVIE and  an example set of  sufficiently scattered samples for $r=2$. These samples clearly do not contain the vertices of $\Ball_{1,+}$.  The convex hull of the samples contains the nonspherical  MVIE and is tangent to its boundary at the points where the polytope is tangent. 
The picture corresponding to the polars of the sets in Figure   \ref{fig:sparsenonnegsuffsc2D}(a) is provided in Figure \ref{fig:sparsenonnegsuffsc2D}(b). The polar of the polytope $\Ball_{1,+}$ is a triangular region, the  boundary of which is shown with dashed-red lines. We note that as expected from (PMF.SS.ii), $\text{conv}(\bS)^{*,\frac{1}{r+1}\mathbf{1}}$ intersects the boundary of $\mathcal{E}_{\Ball_{1,+}}^{*,\frac{1}{r+1}\mathbf{1}}$ only at the vertices of  $\Ball_{1,+}^{*,\frac{1}{r+1}\mathbf{1}}$. 

Figure \ref{fig:sparsenonnegsuffsc3D} illustrates a sufficiently scattered distribution of columns of $\bS$ for the three-dimensional case.
\begin{figure}[H]
\centering
\begin{subfigure}[b]{0.27\textwidth}
\includegraphics[width=\textwidth,trim={0.5cm 0.5cm 0 0.85cm},clip]{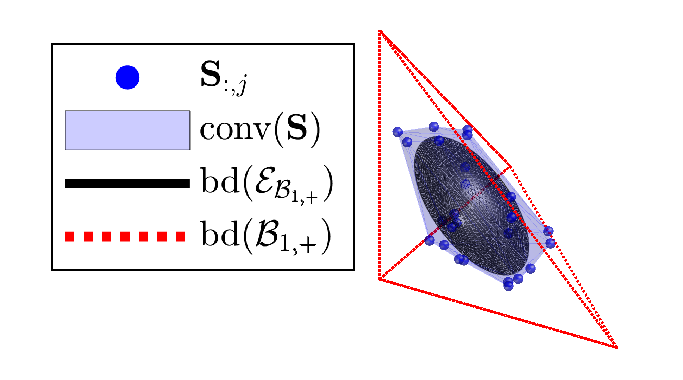}
\end{subfigure}
\caption{$3$D  Sufficiently scattered case example.}
\label{fig:sparsenonnegsuffsc3D}
\end{figure}

We characterize the identifiability condition for the nonnegative sparse PMF problem through the following theorem.
\begin{theorem}
Given the general PMF setting outlined in Section \ref{sec:PMFproblem}, if $\bS_g$ is a sufficiently scattered factor for {\it sparse nonnegative PMF} according to Definition \ref{suffscat}, then all global optima $\bH_*,\bS_*$ of the  \it Det-Max optimization problem in (\ref{eq:detmaxoptimization}) for $\Pcal=\Ball_{1,+}$ satisfy
\begin{eqnarray}
\bH_*&=&\bH_g\bPi^T,\\
\bS_*&=&\bPi\bS_g,
\end{eqnarray}
where $\bPi\in\mathbb{R}^{r \times r}$ is a permutation matrix.
\end{theorem}
\begin{proof}
Following the same arguments in the proof of Theorem \ref{thm1}, we start with the equivalent determinant maximization formulation for the nonnegative sparse PMF:
\begin{maxi!}[l]<b>
{\bA\in \mathbb{R}^{r \times r}}{|\det(\bA)|\label{eq:nonsparseeqobjective}}{\label{nonsparseeqproblem}}{}
\addConstraint{0 \le\bA_{i,:}{\bS_g}_{:,j}} {\le 1,\quad}{i=1, \ldots, r \label{eq:eqnonsparseconstr1}}
\addConstraint{}{}{j=1,\ldots, N \nonumber}
\addConstraint{\mathbf{1}^T\bA{\bS_g}_{:,j}}{\le 1}{j=1,\ldots, N. \label{eq:eqnonsparseconstr2}}
\end{maxi!}
The proof consists of four steps.

\uline{The first step:} \textit{ We use the polytopic  constraints  (\ref{eq:eqnonsparseconstr1}) and (\ref{eq:eqnonsparseconstr2}), and the sufficient scattering condition (PMF.SS.i) to show that optimal $\bA_*$ is a real orthogonal matrix with $\mathbf{1}^T\bA_*\mathbf{1}=r$ and $\|{\bA_*}_{i,:}\|_2={\bA_*}_{i,:}\mathbf{1}=1$ for all $i=1, \ldots, r$.} 

Based on the sufficient scattering condition (PMF.SS.i), which is $\text{conv}(\bS_g)\supset \mathcal{E}_{\Ball_{1,+}}$ (as illustrated in Figures \ref{fig:sparsenonnegsuffsc2D}(a) and \ref{fig:sparsenonnegsuffsc3D}), any member of the MVIE should satisfy  the  constraints  (\ref{eq:eqnonsparseconstr1}) and (\ref{eq:eqnonsparseconstr2}). Therefore, we have
\begin{eqnarray}
 0 \le\bA_{i,:}\bC_{\Ball_{1,+}}\bu+\frac{1}{r+1}\bA_{i,:}\mathbf{1}\le1, \hspace{0.1in} \forall \bu \in \Ball_2, \label{ineq:sparsenonnegpoly1}\\
 \mathbf{1}^T\bA\bC_{\Ball_{1,+}}\bu+\frac{1}{r+1}\mathbf{1}^T\bA\mathbf{1}\le1, \hspace{0.1in} \forall \bu \in \Ball_2, \label{ineq:sparsenonnegpoly2}
 \end{eqnarray}
Specifically, if we substitute $\bu=-\frac{(\bA_{i,:}\bC_{\Ball_{1,+}})^T}{\|\bA_{i,:}\bC_{\Ball_{1,+}}\|_2}$ in (\ref{ineq:sparsenonnegpoly1}), we obtain
$\|\bA_{i,:}\bC_{\Ball_{1,+}}\|_2\le\frac{1}{r+1}\bA_{i,:}\mathbf{1}$,  
which further implies
\begin{eqnarray}
\bA_{i,:}\bC_{\Ball_{1,+}}^2(\bA_{i,:})^T\le\left(\frac{1}{r+1}\bA_{i,:}\mathbf{1}\right)^2, \label{ineq:sparsenonneginc12n}
 \end{eqnarray}
for all $i=1, \ldots, r$. Inserting  (\ref{parameter1}) for $\bC_{\Ball_{1,+}}^2$ (derived in Appendix \ref{app:mviesparsenonnegative}) in (\ref{ineq:sparsenonneginc12n}), we obtain
\begin{eqnarray*}
\frac{1}{r(r+1)}\left(\|\bA_{i,:}\|_2^2-\frac{(\bA_{i,:}\mathbf{1})^2}{r+1} \right)\le \frac{(\bA_{i,:}\mathbf{1})^2}{(r+1)^2}.
 \end{eqnarray*}
Simplifying this expression, we have
\begin{eqnarray}
\|\bA_{i,:}\|_2^2\le (\bA_{i,:}\mathbf{1})^2, \label{ineq:constraint1}
\end{eqnarray}
for all $i=1, \ldots, r$. If we substitute   $\bu=\mathbf{0}$ in (\ref{ineq:sparsenonnegpoly1}), we obtain $\bA_{i,:}\mathbf{1}\ge 0$. Therefore, we can rewrite (\ref{ineq:constraint1}) as  $\|\bA_{i,:}\|_2\le \bA_{i,:}\mathbf{1}$.
Furthermore, if we substitute $\bu=\frac{1}{\sqrt{r}}\mathbf{1}$ in (\ref{ineq:sparsenonnegpoly2}), we have
\begin{eqnarray}
\frac{1}{\sqrt{r}}\mathbf{1}^T\bA\bC_{\Ball_{1,+}}\mathbf{1}+\frac{1}{r+1}\mathbf{1}^T\bA\mathbf{1}\le1. \label{ineq:sparsenonneginc22}
\end{eqnarray}
Inserting (\ref{eq:mvienonsparsecase}) for $\bC_{\Ball_{1,+}}$ in (\ref{ineq:sparsenonneginc22}), and after applying some simplifications, we obtain 
\begin{eqnarray}
\mathbf{1}^T\bA\mathbf{1}\le r. \label{ineq:constraint2}
\end{eqnarray}
Replacing the constraints of (\ref{nonsparseeqproblem}) with 
(\ref{ineq:constraint1}) and (\ref{ineq:constraint2}), we obtain the following optimization problem, the solution of which provides an upper bound for (\ref{nonsparseeqproblem}):
\begin{maxi!}[l]<b>
{\bA\in \mathbb{R}^{r \times r}}{|\det(\bA)|\label{eq:nonsparseequbobjective}}{\label{nonsparseequbproblem}}{}
\addConstraint{\bA_{i,:}\mathbf{1}} {\ge \|\bA_{i,:}\|_2,\quad}{i=1, \ldots, r \label{eq:eqnonsparseubconstr1}}
\addConstraint{\mathbf{1}^T\bA\mathbf{1}}{\le r.}{\label{eq:eqnonsparseubconstr2}}
\end{maxi!}
The optimal value of the objective function (\ref{eq:nonsparseequbobjective}) is bounded from above by
\begin{eqnarray}
|\det(\bA)|&\le&  \|\bA^T_{1,:}\|_2 \|\bA^T_{2,:}\|_2 \ldots \|\bA^T_{r,:}\|_2, \label{ineq:hadamardsparsenonneg} \\
&\le&  (\bA_{1,:}\mathbf{1}) (\bA_{2,:}\mathbf{1}) \ldots (\bA_{r,:}\mathbf{1}), \label{ineq:normsparsenonneg} \\
&\le&  (\frac{1}{r}\mathbf{1}^T\bA\mathbf{1})^r, \label{arithm-geo}\\
&\le&  1, \label{upper-bound}
\end{eqnarray}
where the expression in (\ref{ineq:hadamardsparsenonneg}) is Hadamard's inequality, (\ref{arithm-geo}) is due to the arithmetic-geometric mean inequality, and (\ref{upper-bound}) is due to (\ref{eq:eqnonsparseubconstr2}). Therefore, for any optimal solution  $\bA_*$ of (\ref{nonsparseeqproblem}), we have
\begin{eqnarray}
|\det(\bA_*)|\le 1. \label{upper-bound2}
\end{eqnarray}
The equality in (\ref{upper-bound2}) is achieved if and only if $\bA_*$ is a real orthogonal matrix with
\begin{eqnarray}
\mathbf{1}^T\bA_*\mathbf{1}=r,\hspace{0.01in} \text{ and } \|{\bA_*}_{i,:}\|_2={\bA_*}_{i,:}\mathbf{1}=1, \hspace{0.01in}  i=1, \ldots, r \label{ubiffcond}.
\end{eqnarray}
Furthermore, since the identity matrix, $\mathbf{I}$, is clearly a feasible point for (\ref{nonsparseeqproblem}), its determinant corresponds to a lower bound. Therefore,  for any optimal solution $\bA_*$ of  (\ref{nonsparseeqproblem}), we obtain
\begin{eqnarray}
 |\det(\bA_*)|\ge 1. \label{lower-bound}
\end{eqnarray}
Combining  the lower bound in (\ref{lower-bound}) and the upper bound in (\ref{upper-bound2}) and with the equality conditions in (\ref{ineq:hadamardsparsenonneg})--(\ref{upper-bound}), we conclude that $\bA_*$ is an optimal solution for the problem in (\ref{nonsparseeqproblem})  only if $\bA_*$ is a real orthogonal matrix satisfying (\ref{ubiffcond}).

\uline{The second step:} \textit{ We show that the scaled rows of the global optima of (\ref{nonsparseeqproblem}) are in $\text{conv}(\bS_g)^{*,\frac{1}{r+1}\mathbf{1}}$.}

A global optimal point $\bA_*$ satisfies (\ref{eq:eqnonsparseconstr1}), therefore, we can write ${\bA_*}_{i,:}{\bS_g}_{:,j}\ge 0$ for all possible $(i,j)$ pairs. 
 Multiplying this expression by $-(r+1)$, we obtain $-(r+1){\bA_*}_{i,:}{\bS_g}_{:,j}\le 0$. From the previous step, we also have ${\bA_*}_{i,:}\mathbf{1}=1$. Combining these expressions, we can write
\begin{eqnarray}
-(r+1){\bA_*}_{i,:}\left({\bS_g}_{:,j}-\frac{1}{r+1}\mathbf{1}\right)&\le& 1, \nonumber
\end{eqnarray}
which implies $-(r+1)({\bA_*}_{i,:})^T \in \text{conv}(\bS_g)^{*,\frac{1}{r+1}\mathbf{1}}$ for all $i\in \{1, \ldots, r\}$. We will  show in the fourth step that these scaled rows of $\bA_*$ are also the vertices of $\Ball_{1,+}^{*,\frac{1}{r+1}\mathbf{1}}$.

Furthermore, scaling (\ref{eq:eqnonsparseconstr2}) with $(r+1)$, we can write $((r+1)\bA_*^T\mathbf{1})^T{\bS_g}_{:,j}\le r+1$ for all $j \in \{1, \ldots, N\}$. Replacing $r$ on the right with  $\mathbf{1}^T\mathbf{A}_*\mathbf{1}$ (based on (\ref{ubiffcond})), we have
\begin{eqnarray}
((r+1)\bA_*^T\mathbf{1})^T{\bS_g}_{:,j}\le \mathbf{1}^T\mathbf{A}_*\mathbf{1}+1. \label{eq:otherterm}
\end{eqnarray}
Reorganizing (\ref{eq:otherterm}), leads to
\begin{eqnarray*}
((r+1)\bA_*^T\mathbf{1})^T\left({\bS_g}_{:,j}-\frac{1}{r+1}\mathbf{1}\right)\le 1,
\end{eqnarray*} 
from which we conclude
 $(r+1){\bA_*}^T\mathbf{1}\in \text{conv}(\bS_g)^{*,\frac{1}{r+1}\mathbf{1}}$. 

\uline{The third step:} \textit{ We  show that the $-(r+1)$ scaled rows of the global optima $\bA_*$ of (\ref{nonsparseeqproblem}) are on  the boundary of the polar ellipsoid $\mathcal{E}_\Pcal^{*,\frac{1}{r+1}\mathbf{1}}$.}

From Appendix \ref{app:convex}, the polar of the ellipsoid in (\ref{def:ellipsoid}) is given by
\begin{eqnarray*}
\mathcal{E}^{*,\bg_P}_{\mathcal{P}}=\{\bC_P^{-1}\bu \mid 
  \|\bu\|_2\le 1,  \bu\in\mathbb{R}^r\},
\end{eqnarray*}
the boundary of which  can be written as 
\begin{eqnarray}
\text{bd}(\mathcal{E}_\mathcal{P}^{*,\bg_P})=\{\bx \mid \|\bx^T\bC_P\|_2=1\}. \label{eq:polarelboundary}
\end{eqnarray}
To check whether $-(r+1)({\bA_*}_{i,:})^T$ is on the boundary of the polar of the MVIE, we evaluate the norm expression in (\ref{eq:polarelboundary}):
\begin{eqnarray}
\|(r+1){\bA_*}_{i,:}\bC_{\Ball_{1,+}}\|_2^2=&(r+1)^2{\bA_*}_{i,:}\bC_{\Ball_{1,+}}^2{\bA_*}_{i,:}^T,\nonumber\\
=&\frac{r+1}{r}\left(\|{\bA_*}_{i,:}\|_2^2-\frac{({\bA_*}_{i,:}\mathbf{1})^2}{r+1}\right), \label{forbound1}\\
=& 1. \label{forbound2}
\end{eqnarray}
We used  (\ref{parameter1}) to obtain (\ref{forbound1}), and (\ref{ubiffcond}) to simplify (\ref{forbound1}) into (\ref{forbound2}). As a result, $-(r+1)({\bA_*}_{i,:})^T\in \text{bd}(\mathcal{E}_\mathcal{P}^{*,\bg_P})$. Therefore, the rows of the global optimum $\bA_*$ lie in the boundary of the polar domain ellipsoid illustrated in Figure \ref{fig:sparsenonnegsuffsc2D}(b).

Similarly, we can  show $(r+1){\bA_*}^T\mathbf{1}\in \text{bd}(\mathcal{E}_\mathcal{P}^{*,\bg_P})$ through
\begin{eqnarray}
\|(r+1)\mathbf{1}^T{\bA_*}\bC_{\Ball_{1,+}}\|_2^2&=&(r+1)^2\mathbf{1}^T{\bA_*}\bC_{\Ball_{1,+}}^2{\bA_*}^T\mathbf{1},\nonumber\\
&&\hspace*{-1in}=\frac{r+1}{r}\left(\mathbf{1}^T{\bA_*}{\bA_*}^T\mathbf{1}-\frac{(\mathbf{1}^T{\bA_*}\mathbf{1})^2}{r+1}\right), \label{forbound11}\\
&&\hspace*{-1in}=1. \label{forbound21}
\end{eqnarray}
We again used  (\ref{parameter1}) to obtain (\ref{forbound11}), and (\ref{ubiffcond}) to simplify (\ref{forbound11}) into (\ref{forbound21}).

\uline{The fourth step:} \textit{ We combine the results of the previous steps and (PMF.SS.ii) to show that all global optima for (\ref{nonsparseeqproblem}) 
are permutation matrices.}

Combining the results of  the second and third steps of the proof, we deduce that $-(r+1)({\bA_*}_{i,:})^T\in \text{conv}(\bS_g)^{*,\frac{1}{r+1}\mathbf{1}}\cap\text{bd}(\mathcal{E}_\Pcal^{*,\frac{1}{r+1}\mathbf{1}})$ for all $i \in \{1, \ldots ,r\}$ and $(r+1){\bA_*}^T\mathbf{1}\in \text{conv}(\bS_g)^{*,\frac{1}{r+1}\mathbf{1}}\cap\text{bd}(\mathcal{E}_\Pcal^{*,\frac{1}{r+1}\mathbf{1}})$. 
The sufficient scattering condition (PMF.SS.ii) dictates that $\text{conv}(\bS_g)^{*,\frac{1}{r+1}\mathbf{1}} \cap \text{bd}(\mathcal{E}_{\Ball_{1,+}}^{*,\frac{1}{r+1}\mathbf{1}})=\text{ext}(\Ball_{1,+}^{*,\frac{1}{r+1}\mathbf{1}})$.   Due to the canonical description in (\ref{eq:B1pluspolar}), 
\begin{eqnarray}
\text{ext}(\Ball_{1,+}^{*,\frac{1}{r+1}\mathbf{1}}) = \{-(r+1)\be_1, \ldots,-(r+1)\be_r, (r+1)\mathbf{1}\}.\label{eq:extB1pluspolar}
\end{eqnarray}
This condition is visible in Figure \ref{fig:sparsenonnegsuffsc2D}(b)  as a triangle intersecting the boundaries of the ellipsoid and $\text{conv}(\bS_g)^{*,\frac{1}{r+1}\mathbf{1}}$ at its $(r+1)$ vertices only.
Therefore, we conclude that the set $\{(r+1)\bA_*^T\mathbf{1},-(r+1)({\bA_*}_{i,:})^T,i\in\{1,\ldots,r\}\}$ are all vertices of $\Ball_{1,+}^{*,\frac{1}{r+1}\mathbf{1}}$ given by (\ref{eq:extB1pluspolar}). This equivalence together with (\ref{ubiffcond}) lead to the conclusion that the rows of $\bA_*$ are the standard basis vectors; therefore, $\bA_*$ is a permutation matrix.\end{proof}

\section{Generalized PMF}
\label{sec:generalizedpmf}
In the previous section, we concentrated on four particular polytope examples that correspond to the cross combinations of sparse/antisparse and nonnegative/signed attributes. For  these examples, the resulting features are globally applied to all latent vector components. We showed that PMF settings based on these particular polytopes are  always  identifiable using Det-Max optimization problem in (\ref{eq:detmaxoptimization})  if the sufficient scattering condition in Definition \ref{suffscat} holds.

It is interesting to explore whether we can develop alternative polytopes that lead to Det-Max identifiable PMF generative settings as defined in Definition \ref{def:detmaxident}. In this section, we provide a positive answer and show that the set of  polytopes that qualify for the PMF framework is infinitely rich \cite{tatli:2021icassp}. In particular, we show that, as long as a polytope conforms with a specific symmetry restriction, it would always  lead to identifiable generative models under the sufficiently scattering assumption. We now formalize the definition of ``identifiable polytopes".
\begin{definition}{\it Identifiable Polytopes with respect to the Det-Max Criterion}: We refer to a polytope $\Pcal$ as ``Det-Max identifiable" (or simply ``identifiable")  if all generative models in (\ref{eq:PMFgen}) and (\ref{eq:PMFgen2}) based on sufficiently scattered samples from $\Pcal$, according to Definition \ref{suffscat}, are ``Det-Max identifiable"
\end{definition}
We show that the identifiable polytopes should satisfy a particular symmetry condition, which is laid out in the following definition:
\begin{definition}\label{def:permsign}{\it Permutation-and/or-Sign-Only Invariant Set.}  A set $\mathcal{F}$ is called  permutation-and/or sign-only invariant if and only if any linear transformation that satisfies
\begin{eqnarray*}
\mathbf{A}(\mathcal{F})=\mathcal{F}
\end{eqnarray*}
has the form $\mathbf{A}=\mathbf{\Pi}\mathbf{D}$, where $\mathbf{\Pi}\in\mathbb{R}^{r \times r}$ is a permutation matrix, and $\mathbf{D}\in\mathbb{R}^{r \times r}$ is a full rank diagonal matrix with its diagonal entries in $\{-1,1\}$.
\end{definition}
We note that Definition \ref{def:permsign} defines a symmetry restriction: a set satisfying the condition in this definition  can not be mapped to itself under any linear transformation other than the combination of permutation-sign scaling transformations. The following theorem, the proof of which is provided in  \cite{tatli:2021icassp} and Appendix \ref{app:proofoftheorem}, characterizes all ``Det-Max identifiable" polytopes based on the symmetry restriction in Definition \ref{def:permsign}  \cite{tatli:2021icassp}:
\begin{theorem}[\it Det-Max Identifiable Polytope] \label{thm:generalized}
A polytope $\Pcal$ is ``Det-Max identifiable" if and only if its set of vertices, $ext(\Pcal)$, is a permutation-and/or-sign-only invariant set.
\end{theorem}
The symmetry condition imposed by this theorem is satisfied by infinitely many polytopes. The abundance of polytope choices implies a degree of freedom  for defining a diverse set of feature descriptions for latent vectors. In particular, this diversity can be exploited to render latent vector features with heterogeneous structures without resorting to any stochastic assumption such as independence. This property contrasts with the existing deterministic matrix factorization frameworks, such as NMF, SCA, and BCA, which impose a common attribute, such as nonnegativity, antisparsity or sparsity, over the whole vector.  Using the PMF framework, it is possible to choose only a subset of the components to be nonnegative. Furthermore, we can   impose sparsity constraints on  potentially overlapping multiple subsets of components. In the numerical examples section (Section \ref{sec:polylocalfeatures}), we provide an example of such a heterogeneous latent vector design.

\section{Algorithm}
\label{sec:algorithm}
The main emphasis of the current article is laying out the PMF framework and the corresponding identifiability analysis. To illustrate its use, we adopt the iterative algorithm in  \cite{fu2016robust} which is originally proposed for the SSMF framework.

We start by introducing the Determinant Minimization Det-Min problem, equivalent to  the Det-Max optimization problem in (\ref{eq:detmaxoptimization}) under the equality constraint in (\ref{eq:detmaxconstr1}) \cite{fu2019nonnegative}:
\begin{mini!}[l]<b>
{\bH\in \mathbb{R}^{M \times r},\bS\in \mathbb{R}^{r\times N}}
{\det(\bH^T\bH)\label{eq:detminobjective}}{\label{eq:detminoptimization}}{}
\addConstraint{\bY=\bH \bS}{\label{eq:detminconstr1}}{}
\addConstraint{\bS_{:,j} \in \mathcal{P}}{,\quad}{j= 1, \ldots, N.\label{eq:detminconstr2}}{ }
\end{mini!}
Similar to  \cite{fu2016robust}, we employ  the Lagrangian optimization,
\begin{mini!}[l]<b>
{\bH,\bS}
{\|\bY-\bH\bS\|_F^2+\lambda \log\det(\bH^T\bH+\tau\bI)\label{eq:detminlobjective}}{\label{eq:detminloptimization}}{}
\addConstraint{\bS_{:,j} \in \mathcal{P}}{,\quad}{j= 1, \ldots, N.\label{eq:detminlconstr2}}{ }
\end{mini!}
corresponding to (\ref{eq:detminoptimization}), where $\tau>0$ is a hyperparameter to ensure that the objective function is bounded from below. 
\algnewcommand\algorithmicinput{\textbf{Input:}}
\algnewcommand\algorithmicoutput{\textbf{Output:}}
\algnewcommand\Input{\item[\algorithmicinput]}%
\algnewcommand\Output{\item[\algorithmicoutput]}%
\begin{algorithm}[H]
\caption{Det-Min algorithm for PMF}
    \begin{algorithmic}[1]
    \Input $\bY$; $r$; Initial $\bH,\bS$;  $\tau$.\\
     $t=0$;\\
     $\bX^{(t)}=\bS$,$\bF^{(t)}=\bI$,$\bH^{(t)}=\bH$,$\bS^{(t)}=\bS$, $q^{(t)}=1$;
    \Repeat\\
       select $L^{(t)}$;\\ 
       $\bS_{:,l}^{(t+1)}\gets P_\Pcal\left(\bX_{:,l}^{(t)}-(\bH^{(t)})^T(\bY_{:,l}-\bH^{(t)}\bX_{:,l}^{(t)})\right)$ for $l=1, \ldots, N$; \\
       $q^{(t+1)}\gets \frac{1+\sqrt{1+(q^{(t)})^2}}{2}$;\\
       $\bX_{:,l}^{(t+1)}\gets \bS_{:,l}^{(t+1)}+\frac{q^{(t)}-1}{q^{(t+1)}}(\bS^{(t+1)}_{:,l}-\bS^{(t)}_{:,l})$  for $l=1, \ldots, N$;\\
       $\bH^{(t+1)}\gets\bY(\bS^{(t+1)})^T(\bS^{(t+1)}(\bS^{(t+1)})^T+\lambda\bF^{(t)})$;\\
       $t \gets t+1$;\\
       $\bF^{(t)}\gets((\bH^{(t)})^T\bH^{(t)}+\tau \bI)^{-1}$.
    \Until some stopping criterion is reached.
    \Output $\bH^{(t)}$;$\bS^{(t)}$.
    \end{algorithmic}
\end{algorithm}
The corresponding steps are provided in Algorithm 1, which is the algorithm in \cite{fu2016robust}, except that the projection onto  the unit simplex is replaced with the projection to the polytope, $P_\Pcal(\cdot)$. The following are examples of this projection  operator:
\begin{itemize}
\item  {\it Antisparse Case:} ${\bX}={P}_{\Ball_\infty}(\bar{\bX})$ defines an elementwise projection operator to $\Ball_\infty$, which can be written as
\begin{eqnarray*}
    {X}_{ij}=\left\{\begin{array}{cc} {\bar{X}}_{ij} & \text{if } |\bar{X}_{ij}|<1, \\
    \text{sign}(\bar{X}_{ij}) & \text{otherwise.} \end{array} \right.
    \end{eqnarray*}
\item {\it Sparse Case}: For ${P}_{\Ball_1}(\cdot)$, the projection onto the $\ell_1$-norm-ball has no closed-form solution; however, efficient iterative algorithms, such as   \cite{duchi2008efficient}, exist.
\item {\it Antisparse Nonnegative Case:} The projection operator is a simple modification of ${P}_{\Ball_\infty}(\cdot)$, where elementwise projections are performed over  $[0,1]$  instead of $[-1,1]$.
\item {\it Sparse Nonnegative Case:} We can simplify the projection operator $\Ball_1$ to obtain the projection onto $\Ball_{1,+}$ (see \cite{duchi2008efficient}).
\end{itemize}

\section{Numerical Experiments}
\label{sec:numexperiments}
\subsection{Polytope with Local Features}
\label{sec:polylocalfeatures}
To illustrate the feature shaping flexibility provided by the PMF framework, we consider the following example polytope:
\begin{eqnarray*}
\def\arraystretch{1.6}
\Pex=\left\{\mathbf{x}\in \mathbb{R}^3\ \middle\vert \begin{array}{l}   x_1,x_2\in[-1,1],x_3\in[0,1],\\ \left\|\left[\begin{array}{c} x_1 \\ x_2 \end{array}\right]\right\|_1\le 1,\left\|\left[\begin{array}{c} x_2 \\ x_3 \end{array}\right]\right\|_1\le 1 \end{array}\right\},
\end{eqnarray*}
which corresponds to  the following local attributes:
\begin{itemize}
\item $x_3$ is nonnegative and $x_1,x_2$ are signed; and
\item $\left[\begin{array}{cc} x_1 & x_2 \end{array}\right]^T$ and $\left[\begin{array}{cc} x_2 & x_3\end{array}\right]^T$ are sparse subvectors.
\end{itemize}
\begin{figure}[ht]
         \centering
         \includegraphics[trim={1.5cm 0 2.2cm 3cm},clip,width=0.27\textwidth]{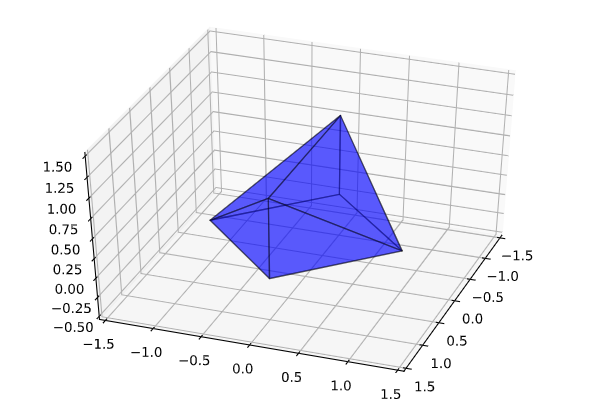}
\caption{Polytope $\Pcal_{\text{ex}}$.}
         \label{fig:Pex}
\end{figure}
The polytope $\Pex$, shown in Figure \ref{fig:Pex}, has $6$ vertices placed in the columns  of the following matrix:
\begin{eqnarray*}
\mathbf{V}_{\Pcal_{\text{ex}}}=\left[\begin{array}{cccccc} 1 & -1 & 0 & 0 & 1 & -1\\
                                       0 &  0 & 1 & -1& 0 &  0\\
                                       0 & 0 & 0  &  0 & 1 & 1 \end{array}\right].
\end{eqnarray*}
Checking all of the possible permutations of the columns of this matrix reveals that the vertex set satisfies the desired symmetry  restriction; therefore, $\Pcal_{\text{ex}}$ is identifiable. Due to its factorial complexity, this check can become computationally demanding for polytopes with a large number of vertices. To illustrate both the identifiability of $\Pcal_{\text{ex}}$ and the convergence behavior of the algorithm for sufficiently scattered samples, we conducted the following experiment. With each run of the experiment, we generate a  set of sufficiently scattered samples from $\Pcal_{\text{ex}}$ using the procedure described below.
\begin{itemize}
\item[S1.]\uline{Generate the polar domain set   $\text{conv}(\bS)^{*,\bg_{\mathcal{P}_{\text{ex}}}}$ in {\bf V-Form}} (\ref{eq:ptope2}): For this purpose we generate $L$ samples from  $\text{conv}(\bS)^{*,\bg_{\mathcal{P}_{\text{ex}}}}$, the columns of $\bK\in\mathbb{R}^{r \times L}$, and define $\text{conv}(\bS)^{*,\bg_{\mathcal{P}_{\text{ex}}}}=\text{conv}(\bK)$: 

\item[i.] According to the condition (PMF.SS.ii) in Definition \ref{suffscat}, the elements of $\text{ext}(\mathcal{P}^{*, \bg_{\mathcal{P}_{\text{ex}}}})$ should be the vertices of $\text{conv}(\bS)^{*,\bg_{\mathcal{P}_{\text{ex}}}}$. Therefore, we first include elements of $\text{ext}(\mathcal{P}_{\text{ex}}^{*, \bg_{\mathcal{P}_{\text{ex}}}})$ in $\bK$, by setting: \begin{eqnarray*}
\bK_{:,1:7}=\left[\begin{array}{ccccccc} 1& 1 & -1 & -1 & 0 & 0 & 0\\
                                1&       -1 &  1 & -1 & 1.6& -1.6 &  0\\
                                     0 &  0 & 0 & 0  &  1.6 & 1.6 & -8/3 \end{array}\right].
\end{eqnarray*}
Note that $\text{conv}(\bK_{,1:7})=\mathcal{\Pcal_{\text{ex}}}^{*, \bg_{\mathcal{P}_{\text{ex}}}}$. If we set $L=7$, and therefore, skip the next step (S1.ii), our procedure would generate the vertices of $\Pcal_{\text{ex}}$, which is a sufficiently scattered set.
\item[ii.] According to (PMF.SS.ii), the remaining vertices of $\text{conv}(\bS)^{*,\bg_{\mathcal{P}_{\text{ex}}}}$ should be in the interior of $\mathcal{E}_{\Pcal_{\text{ex}}}^*$. Therefore, we generate $L-7$ random points in the interior of $\mathcal{E}_{\Pcal_{\text{ex}}}^*$.  For this purpose, we generate $L-7$ i.i.d. $r$-dimensional random samples in $0.9\Ball_2$. Then we multiply these vectors with $\bC_P^{-1}$ to obtain the remaining columns of $\bK$.
\item[iii.] We find {\bf V-Form} for $\text{conv}(\bS)^{*,\bg_{\mathcal{P}_{\text{ex}}}}$ by applying a numerical convex hull algorithm (such as ConvexHull function of Python's Scipy library \cite{2020SciPy-NMethetal}) to the columns of $\bK$. The output of this step is the matrix $\bV\in\mathbb{R}^{r\times L'}$ containing the vertices of  $\mathcal{P}^{*, \bg_{\mathcal{P}_{\text{ex}}}}$.
\item[S2.]\uline{Convert the representation of $\text{conv}(\bS)^{*,\bg_{\mathcal{P}_{\text{ex}}}}$ from {\bf V-Form} to {\bf H-Form}:} 
We convert the representation of $\text{conv}(\bS)^{*,\bg_{\mathcal{P}_{\text{ex}}}}$ from 
{\bf V-Form} in (\ref{eq:ptope2}) to {\bf H-Form} in (\ref{eq:ptope1}). For this purpose, we use the PYPOMAN (a PYthon module for POlyhedral MANipulations) software package (\texttt{duality.compute\_polytope\_halfspaces} function.)  \cite{Pypoman:2020}. The output of this stage are the hyperplane parameters  $(\ba_i,b_i), i=1, \ldots, f$,   for the {\bf H-Form}.
\item[S3.] \uline{Find the vertices of $\text{conv}(\bS_g)$:}  We first  perform the normalization on the hyperplane parameters to obtain $(\ba_i/b_i,1), i=1, \ldots, f$. Note that according to the polar conversion described in  Appendix \ref{app:convex}, obtaining {\bf V-Form} for $\Pcal_{\text{ex}}$ in (\ref{appApolar}) from the {\bf H-Form} for  $\text{conv}(\bS)^{*,\bg_{\mathcal{P}_{\text{ex}}}}$ in (\ref{appApoly}), the vectors $\{\ba_i/b_i+\bg_{\mathcal{P}_{\text{ex}}}, i=1, \ldots, f\}$ are the vertices of $\text{conv}(\bS_g)$, where $\bg_{\mathcal{P}_{\text{ex}}}=[0,0,0.375]$. Therefore, we can set
\begin{eqnarray*}
{\bS_g}_{:,1:f}=\left[\begin{array}{cccc} \frac{\ba_1}{b_1}+\bg_{\mathcal{P}_{\text{ex}}} & \frac{\ba_2}{b_2}+\bg_{\mathcal{P}_{\text{ex}}} & \ldots & \frac{\ba_f}{b_f}+\bg_{\mathcal{P}_{\text{ex}}}\end{array}\right].
\end{eqnarray*}
Note that $f$, the number of vertices of $\text{conv
}(\bS_g)$ (or faces of its polar), is a random quantity.
\end{itemize}
In the following experiments, we chose  $L=30$.
The $\bH_g$ matrix is generated as a $4 \times 3$ i.i.d. Gaussian matrix (with zero mean and unity variance). We added zero mean i.i.d. Gaussian noise to the input matrix $\bY=\bH_g\bS_g$.
The projection operation $P_{\Pcal_{\text{ex}}}$ onto $\Pex$ is implemented through $5$ alternating iterations of the following.
\begin{itemize}
\item the projection onto the rectangle corresponding to the range constraints, implemented by elementwise clipping operations;
\item the projection onto the $\ell_1$-norm-ball for $\left[\begin{array}{cc} x_1 & x_2 \end{array}\right]^T$;
\item the projection onto the $\ell_1$-norm-ball for $\left[\begin{array}{cc} x_2 & x_3 \end{array}\right]^T$.
\end{itemize}
For the algorithm hyperparameters, we selected $\tau=10^{-8}$, $L^{(t)}=5\|(\bH^{(t)})^T\bH^{(t)}\|_2$ and  $\lambda=0.01$. 
\begin{figure}[ht]
         \centering
         \includegraphics[trim={0cm 0cm 0.0 0cm},clip,width=0.35\textwidth]{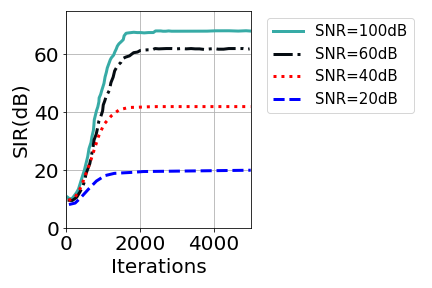}
\caption{Average SIR convergence curves for $\Pcal_{\text{ex}}$.}
         \label{fig:Pex2}
\end{figure}
The convergence of the algorithm in terms of the signal-to-interference ratio (SIR), averaged over $100$ realizations, as a function of the iterations is shown in Figure \ref{fig:Pex2} for different signal-to-noise-ratio (SNR) levels. These experiments confirm the identifiability under the sufficiently scattered condition.

\begin{figure}[ht]
`    \centering
    \begin{minipage}[b]{0.15\textwidth}
    \includegraphics[width=0.99\textwidth,trim={5.0cm 0cm 2.3cm 1.1cm},clip
    ]{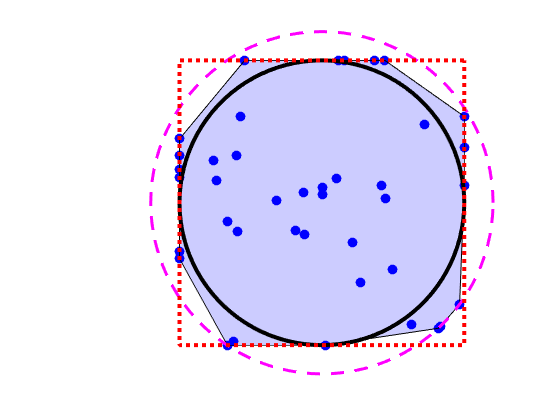}
    \caption*{(a)}
    \label{fig:Pex321}
    \end{minipage}\hfill
    \begin{minipage}[b]{0.25\textwidth}
    \includegraphics[width=0.99\textwidth
    ,trim={0cm 0.0cm 0cm 0.0cm},clip
    ]{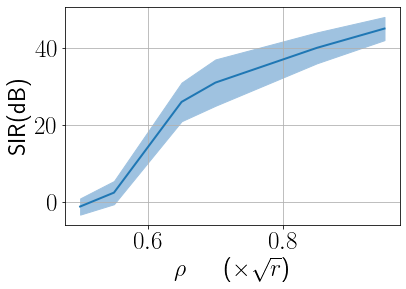}
    \caption*{(b)}
    \label{fig:Pex322}
    \end{minipage}
\caption{(a)  Sample generation example for Section \ref{sec:specialPMFex}: $\Pcal=\Ball_\infty$ ($r=2$), the dashed circle represents the boundary of the inflated MVIE  for $\rho=0.85\sqrt{r}$. (b) SIR (mean-solid line with 
std. envelope) as a function of $\rho$ for $\Ball_\infty$, $r=10$ and $N=500$.}
         \label{fig:Pex32}
\end{figure}
\subsection{Special Polytopes in Section \ref{sec:specialPMF}}
\label{sec:specialPMFex}
In this section, we provide experiments for the special polytopes in Section \ref{sec:specialPMF} and illustrate an alternative sample generation method. For high dimensions, the sufficiently scattered sample generation procedure that we proposed in Section \ref{sec:polylocalfeatures} may  not be feasible due to the computational complexity of the polytope representation conversion step in Step S2. Instead, we use another method which is based on   principles similar to those used in  \cite{lin2015identifiability,fu2016robust}. In this technique, we generate random samples inside the inflated version of the MVIE, and project them onto the polytope. Therefore, the procedure consists of two steps:
\begin{itemize}
\item[S1] \uline{Generate $N$ samples from the inflated version of the MVIE of the polytope}: we first generate $N$ random i.i.d. Gaussian vector samples  $\bw_i\sim \mathcal{N}(\mathbf{0},\frac{\rho^2}{r}\mathbf{I}), i=1,\ldots N$, in $\mathbb{R}^r$, the norms of which  concentrate around the mean $E(\|\bw_i\|_2)=\rho$.  We call $\rho$  ``inflation constant" and choose $\rho>1$. Then, we saturate vectors with norm greater than $\rho$ by defining $\bu_i=\rho\bw_i/\|\bw_i\|_2$ if $\|\bw_i\|_2>\rho$ and $\bu_i=\bw_i$ otherwise. The resulting $\bu_i$ vectors are random samples in the hypersphere $\rho\Ball_2$, i.e., expanded unit hypersphere. Finally, we map the samples $\bu_i$ into the inflated version of  $\mathcal{E}_\Pcal$  using $\bz_i=\bC_\Pcal\bu_i+\bg_\Pcal$. Therefore, the resulting samples lie in $\rho(\mathcal{E}_\Pcal-\bg_\Pcal)+\bg_\Pcal$, the $\rho$-inflated version of the MVIE.
\item[S2.] \uline{Project the samples $\bz_i$ onto the polytope}: i.e., ${\bS_g}_{:,i}=P_\Pcal(\bz_i), i=1, \ldots, N$. 
\end{itemize}
Figure \ref{fig:Pex32}(a) illustrates the proposed sample generation for $\Pcal=\Ball_\infty$ and $r=2$. The inflation constant is selected as $\rho=0.85\sqrt{r}$.
 Note that the choice $\rho=1$ corresponds to the MVIE, and $\rho=\sqrt{r}$ corresponds to the minimum volume enclosing sphere of $\Ball_\infty$. Therefore, when $\rho<\sqrt{r}$, the vertices of $\Ball_\infty$ are not covered by the inflated MVIE.
 
 For the experiments with polytopes in Section \ref{sec:specialPMF}, we took $r=10$ and $M=20$. At each realization, we independently  generated $\bH_g$ as a $20\times 10$ i.i.d. Gaussian matrix with zero mean and unity variance. We used different empirical $\lambda$ parameter choices  for different polytope and sample size selections to improve the SIR performance. We conducted this experiment for $300$ realizations.
  Figure \ref{fig:Pex32}(b) shows the SIR obtained as a function of $\rho$ for $\Pcal=\Ball_\infty$ and $N=500$. In Figure \ref{fig:Pex3}, for the choice $\rho=0.85\sqrt{r}$, we show the SIR as a function of the sample size $N$ for all four polytopes in Section \ref{sec:specialPMF}. 
   Both Figure \ref{fig:Pex32} and \ref{fig:Pex3}  confirm that the sufficiently scattered set generation probability increases with the increasing values of  $\rho$ and $N$ as expected.
\begin{figure}[ht]
         \centering
         \includegraphics[trim={0cm 0cm 0.0 0.1cm},clip,width=0.42\textwidth]{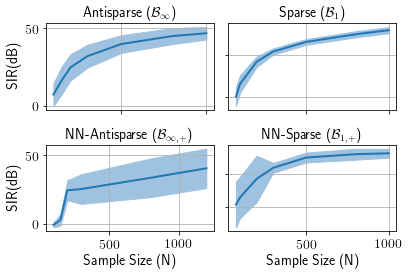}
\caption{SIR (mean-solid line with 
std. envelope) as a function of sample size $N$ for the polytopes in Section \ref{sec:specialPMF}.}
         \label{fig:Pex3}
\end{figure}
\subsection{Sparse Dictionary Learning for Natural Image Patches}
As the second example, we applied sparse PMF to $12 \times 12$ image patches obtained from Olshausen's prewhitened natural images (available at http://www.rctn.org/bruno/sparsenet/). These patches are vectorized (into $144 \times 1$ vectors) and placed in the columns of the $\bY$ matrix. The columns of the $\bH \in \mathbb{R}^{144 \times 144}$ matrix obtained from the sparse PMF algorithm (with $\lambda=1$, $\tau=10^{-6}$ and  $L^{(t)}=4\|(\bH^{(t)})^T\bH^{(t)}\|_2$)  are reshaped as $12 \times 12$ images (rescaled to the $0-1$ range) are shown in Figure \ref{fig:imagePatch}. It is interesting to note that, although we used a different normative approach (based on determinant maximization) than \cite{olshausen1997sparse}, we obtained similar Gabor-like edge features for the natural image patches.
\begin{figure}[ht]
         \centering
         \includegraphics[trim={1 0.8 0.0cm 0.4cm},clip,width=0.4\textwidth]{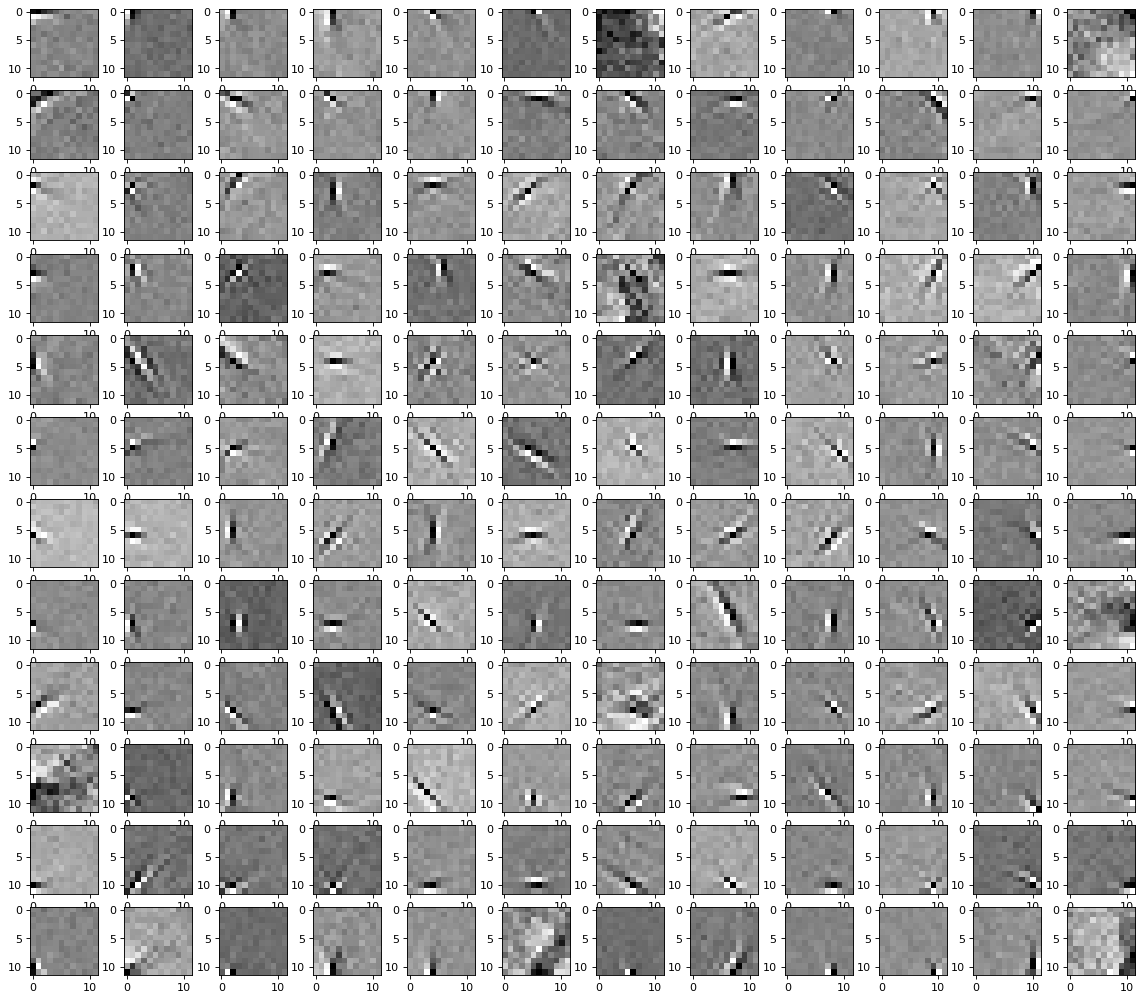}
\caption{Dictionary obtained by sparse PMF for prewhitened natural image patches.}
         \label{fig:imagePatch}
\end{figure}
\section{Conclusion}
\label{sec:conclusion}
In this article, we introduced PMF as a novel  structured matrix factorization approach. Having infinite choices of  identifiable polytopes positions PMF as a general framework with a diverse set of feature representation selections. 
We also proposed a geometric approach for identifiability analysis, which provided practically plausible conditions for the applicability of the PMF framework. We can foresee various future extensions including
underdetermined data models, fully structured matrix factorization, noise/outlier analysis and efficient algorithms.

\appendices
\section{Relevant Convex Analysis Preliminaries}
\label{app:convex}
\par The polar of a set $C\subset \mathbb{R}^r$ with respect to the point $\mathbf{d}\in \mathbb{R}^r$ is defined as
\begin{eqnarray*}
C^{*,\mathbf{d}}=\{ \bx \in \mathbb{R}^r | \langle \bx,\by-\mathbf{d} \rangle \leq1 \ \forall \by\in C\}.
\end{eqnarray*}
When $\mathbf{d}=\mathbf{0}$, we simplify the notation as $C^*$. Based on this notation, we can write $C^{*,\mathbf{d}}=(C-\mathbf{d})^*$. A polytope containing the origin as an interior point can be written as
\begin{eqnarray}
 \Pcal=\{\bx \ | \langle \ba_i, \bx\rangle \leq 1, i=1, \ldots,f\},
 \label{appApoly}
\end{eqnarray}
 the polar of which is given by (Theorem 9.1 in \cite{brondsted2012introduction})
\begin{eqnarray} 
\Pcal^{*}=\text{conv}(\left[\begin{array}{cccc}\ba_1 & \ba_2 & \ldots & \ba_f\end{array}\right]). 
\label{appApolar}
\end{eqnarray}
Therefore, the face normals $\ba_i$ of the polytope $\Pcal$ are the vertices of its polar $\Pcal^{*}$. Note that if $\mathbf{d}$ is an interior point of $\Pcal$, to calculate $\Pcal^{*,\mathbf{d}}$, one can write $\Pcal-\mathbf{d}$ in the form (\ref{appApoly}) and use $\Pcal^{*,\mathbf{d}}=(\Pcal-\mathbf{d})^*$ and (\ref{appApolar}).

For an ellipsoid $\mathcal{E}_P\subset \mathbb{R}^r$ defined by
\begin{eqnarray}
 \mathcal{E}_\mathcal{P}=\{\bC_\mathcal{P}\bu+\bg_\mathcal{P} \mid \|\bu\|_2\le 1, \bu\in\mathbb{R}^r\}, \label{def:ellipsoidapp}
 \end{eqnarray}
where $\bC\in\mathbb{R}^{r \times r}$ and $\bg\in\mathbf{R}^r$, its polar with respect to its center $\bg_p$, can be written as
 \begin{eqnarray}
 \mathcal{E}^{*,\bg_P}_\mathcal{P}&=&\{\bx \mid \langle \bx,\bC_P\bu\rangle \le 1,  \forall \|\bu\|_2\le 1,\bu\in\mathbb{R}^r\}, \nonumber 
 \end{eqnarray}
 which can be simplified to 
\begin{eqnarray}
 \mathcal{E}^{*,\bg_P}_{\mathcal{P}}=\{\bC_P^{-1}\bu \mid 
  \|\bu\|_2\le 1,  \bu\in\mathbb{R}^r\}.
\label{eq:polarellipsoid}
\end{eqnarray}
A particular property of polar operation that we  use  in the article is the reversal of the set inclusion: if  $A\subset B$  holds then we have $B^{*,\mathbf{d}}\subset A^{*,\mathbf{d}}$.

\section{MVIE for Sparse Nonnegative PMF}
\label{app:mviesparsenonnegative}
We can obtain the MVIE parameters $(\bC_{P},\bg_P)$ of the polytope $\Ball_{1,+}$, as the optimal solution of  the following optimization problem: 
\begin{mini!}[l]<b>
{\bC\in\mathbb{R}^{r \times r},\bg\in\mathbb{R}^r}
{-\log\det\bC\label{eq:app_ellipsoid_objective}}{\label{eq:app_ellipsoid_optimization}}{}
\addConstraint{\|\bC\be_i\|_2 -\be_i^T\bg \le 0}{,\quad}{i= 1, \ldots, r\label{eq:app_ellipsoidconstr1}}{ }
\addConstraint{\|\bC\mathbf{1}\|_2 +\mathbf{1}^T\bg \le 1}{\label{eq:app_ellipsoidconstr2}}{}
\addConstraint{\bC\succeq\mathbf{0},}{\label{eq:app_ellipsoidconstr3}}{}
\end{mini!}
which is obtained by applying the description of $\Ball_{1,+}$ in (\ref{ineq:sparsenonnegpoly}) to the generic MVIE optimization problem in (\ref{app:mvie}).
To find the solution of  (\ref{eq:app_ellipsoid_optimization}), we utilize the following Karush-Kuhn-Tucker (KKT) optimality conditions \cite{boyd2004convex} for the solution pair  ($\bC_P,\bg_P$): 
\begin{eqnarray}
-\bC_P^{-1}+\sum_{i=1}^r\lambda_i\frac{\bC_P\be_i\be_i^T}{\|\bC_P\be_i\|_2}+\lambda_{r+1}\frac{\bC_P\mathbf{1}\mathbf{1}^T}{\|\bC_P\mathbf{1}\|_2}-\mathbf{W}=\mathbf{0}, \label{firstcond}\\
\lambda_i(\|\bC_P\be_i\|_2 -\be_i^T\bg_P)=0, \ \text{for all} \ i=1, \ldots, r, \label{eq:bindcon}\\ \lambda_{r+1}(\|\bC_P\mathbf{1}\|_2 +\mathbf{1}^T\bg_P-1)=0, \label{thirdcond}\\
-\sum_{i=1}^r\lambda_i\be_i+\lambda_{r+1}\mathbf{1}=\mathbf{0},\label{lambdaeq}
\\Tr(\mathbf{W}\bC_P)=0, \label{fifthcond}
\end{eqnarray}
where  (\ref{firstcond}) and (\ref{lambdaeq}) represent stationarity conditions of $\bC_P$ and $\bg_P$, respectively. The remaining KKT equations -- (\ref{eq:bindcon}), (\ref{thirdcond}) and (\ref{fifthcond}) -- are known as complementary slackness conditions, which involve optimal dual variables $\lambda_1, \lambda_2, \ldots,\lambda_{r+1}\in\mathbb{R}_+$  and   $\mathbf{W}\succeq\mathbf{0}$ corresponding to the inequality constraints in (\ref{eq:app_ellipsoidconstr1}), (\ref{eq:app_ellipsoidconstr2}) and (\ref{eq:app_ellipsoidconstr3}), respectively.
Based on the optimization in (\ref{eq:app_ellipsoid_optimization}) and the corresponding optimality conditions in (\ref{firstcond})-(\ref{fifthcond}), we can deduce the following: 
\begin{itemize}
\item $\bC_P$ is nonsingular, i.e., $\bC_P\succ\mathbf{0}$, otherwise the objective function $-\log\det\bC_P$ becomes infinite under  the primal feasibility condition  $\bC\succeq\mathbf{0}$  in (\ref{eq:app_ellipsoidconstr3}). Therefore, the nonsingularity of $\bC$ implies $\mathbf{W}=\mathbf{0}$ due  the KKT condition given in  (\ref{fifthcond}). 
\item The condition in (\ref{lambdaeq}) is equivalent to $\lambda_1= \lambda_2= \ldots =\lambda_{r+1}=\lambda$,  further implying that either all or none of the corresponding primal feasibility conditions  are binding, i.e., the inequalities corresponding to these dual variables, (\ref{eq:app_ellipsoidconstr1}) and (\ref{eq:app_ellipsoidconstr2}), are equalities at the optimal point. If we assume none of them are binding, the KKT conditions (\ref{thirdcond}) and (\ref{lambdaeq}) imply that $\lambda_1= \lambda_2= \ldots =\lambda_{r+1}=0$, which further implies  $\bW\neq\mathbf{0}$ due to  (\ref{firstcond}). This outcome contradicts  our earlier finding that $\bW=\mathbf{0}$. Therefore, we conclude that 
\begin{eqnarray}
\lambda_1= \lambda_2= \ldots =\lambda_{r+1}>0, \label{eq:lambdpositive}
\end{eqnarray}
and all corresponding inequalities are binding, i.e.,  $\|\bC_P\be_i\|_2 -\be_i^T\bg_P = 0 \ \text{for} \ i=1, \ldots, r$ and $\|\bC_P\mathbf{1}\|_2 +\mathbf{1}^T\bg_P = 1$ hold.
\item Using (\ref{eq:lambdpositive}) and $\bW=\mathbf{0}$, the expression in (\ref{firstcond}) can be rewritten as
\begin{eqnarray}
-\bC_P^{-1}+\lambda \left( \sum_{i=1}^r\frac{\bC_P\be_i\be_i^T}{\|\bC_P\be_i\|_2}+\frac{\bC_P\mathbf{1}\mathbf{1}^T}{\|\bC_P\mathbf{1}\|_2} \right) = \mathbf{0}. \label{app_c}
\end{eqnarray}
If we multiply both sides of  (\ref{app_c}) by $\bC_P^{-1}$, and rearrange the terms, we obtain
\begin{eqnarray}
\bC_P^{-2}=\lambda \left( \sum_{i=1}^r\frac{\be_i\be_i^T}{\|\bC_P\be_i\|_2}+\frac{\mathbf{1}\mathbf{1}^T}{\|\bC_P\mathbf{1}\|_2} \right). \label{rearrangedC}
\end{eqnarray}
Using the binding constraints   $\|\bC_P\be_i\|_2 =\be_i^T\bg_P$ and $\|\bC_P\mathbf{1}\|_2 = 1-\mathbf{1}^T\bg_P$,  (\ref{rearrangedC}) can be rewritten as
\begin{eqnarray}
\bC_P^{-2}=\lambda \left( \mathbf{G}^{-1}+\frac{\mathbf{1}\mathbf{1}^T}{1-\mathbf{1}^T\bg_P} \right), \label{app_c-2}
\end{eqnarray}
where $\mathbf{G}\in \mathbb{R}^{r \times r}$  is a diagonal matrix, the $i^{\text{th}}$ diagonal entry of which is equal to $\be_i^T\bg$ for all i=1, \ldots, r. Applying the matrix inversion lemma to the right hand side of  (\ref{app_c-2}), we have
\begin{eqnarray}
\bC_P^{2}=\lambda^{-1} \left( \mathbf{G}-\bg_P\bg_P^T \right). \label{c2}
\end{eqnarray}
Multiplying (\ref{c2}) by the ones-vector from both the left and right yields
\begin{eqnarray}
\mathbf{1}^T\bC_P^{2}\mathbf{1}=\lambda^{-1} \mathbf{1}^T\bg_P\left( 1-\mathbf{1}^T\bg_P \right). \label{app_1c21}
\end{eqnarray}
The  binding constraint $\|\bC_P\mathbf{1}\|_2 =1-\mathbf{1}^T\bg_P$ can be used to obtain an alternative expression 
\begin{eqnarray}
\mathbf{1}^T\bC_P^{2}\mathbf{1}=(1-\mathbf{1}^T\bg_P)^2. \label{app_11g2}
\end{eqnarray}
Equating the right hand sides of (\ref{app_1c21}) and (\ref{app_11g2}), we have
\begin{eqnarray}
\lambda = \frac{\mathbf{1}^T\bg_P}{1-\mathbf{1}^T\bg_P} = \frac{1}{1-\mathbf{1}^T\bg_P}-1.  \label{lambda1}
\end{eqnarray}
\item Squaring both sides of the binding constraint $\|\bC_P\be_i\|_2 = \be_i^T\bg$,  we obtain $\be_i^T\bC_P^{2}\be_i = (\be_i^T\bg_P)^2=g_{P,i}^2$, where $g_{P,i}$ stands for $\be_i^T\bg_P$. Inserting the expression for $\bC_P^{2}$ in (\ref{c2}), we obtain
\begin{eqnarray*}
\lambda^{-1}(g_{P,i}-g_{P,i}^2)=g_{P,i}^2,
\end{eqnarray*}
which leads to
\begin{eqnarray}
\lambda = \frac{1-g_{P,i}}{g_{P,i}} = \frac{1}{g_{P,i}}-1, \hspace{0.1in} \forall i \in \{1, \ldots, r\}. \label{lambda2}
\end{eqnarray}
Combining  (\ref{lambda1}) and (\ref{lambda2}), we can write
\begin{equation*}
    1-\mathbf{1}^T\bg = g_{P,i}, \hspace{0.1in} \forall i \in \{1, \ldots, r\}.
\end{equation*}
From this expression and (\ref{lambda2}), we obtain $\bg_P=\frac{1}{r+1}\mathbf{1}$ and $\lambda=r$. Inserting these values into (\ref{c2}) yields
\begin{eqnarray}
\bC_P^{2}= \frac{1}{r}\left( \frac{1}{r+1}\mathbf{I}-\frac{1}{(r+1)^2}\mathbf{1}\mathbf{1}^T \right), \label{parameter1}
\end{eqnarray}
the square root of which is given by
\begin{eqnarray*}
\bC_P= \frac{1}{\sqrt{r}}\left( \frac{1}{\sqrt{r+1}}\mathbf{I}-\frac{\sqrt{r+1}-1}{r(r+1)}\mathbf{1}\mathbf{1}^T \right).
\end{eqnarray*}
\end{itemize}

\section{Proof of Theorem 6}
\label{app:proofoftheorem}
The following lemma from \cite{tatli:2021icassp} is used in the proof of Theorem \ref{thm:generalized}.
\begin{lemma}\underline{\it Polar of the Transformed Set}: \label{lem:polaroftransform}
We can characterize the polar of the transformed set $\bA(S)$ in the following way,
\begin{eqnarray*}
(\bA(S))^{*,\mathbf{d}}
&=&\{ \bx \in \mathbb{R}^n \mid \langle \bx,\bA\by-\mathbf{d} \rangle \leq1 \ \forall \by\in S\}
\end{eqnarray*} 
which corresponds to the set $\bA^{-T}(S^{*,\mathbf{d}})$, when $\bA\mathbf{d} = \mathbf{d}$.
\end{lemma}

{\bf Proof of Theorem \ref{thm:generalized}:} 
Using the same arguments in the proof of Theorem \ref{thm1}, we write the optimization problem equivalent to (\ref{eq:detmaxoptimization}) as
\begin{maxi!}[l]<b>
{\bA\in \mathbb{R}^{r \times r}}{|\det(\bA)|\label{eq:detmaxeqobjective}}{\label{detmaxeqproblem}}{}
\addConstraint{\bA{\bS_g}_{\:,j}} {\in \Pcal, \quad}{j=1, \ldots, N.\label{eq:detmaxeqconstr}}
\end{maxi!}
We first show the ``if" part: Let $\bS_g$ be any sufficiently scattered factor for $\Pcal$ and $\bH_g$ be any full column rank matrix. Let $\bA_*$ represent any global optimum of the equivalent Det-Max optimization problem in (\ref{detmaxeqproblem}). Below we show that $\bA_*(\text{ext}(\Pcal))=\text{ext}(\Pcal)$.
We first note that due to the constraint (\ref{eq:detmaxeqconstr}), we have
\begin{eqnarray}
 \text{conv}(\bA_*\bS_g)\subseteq\Pcal. \label{inc1}
 \end{eqnarray}
Furthermore,
since $\bA=\mathbf{I}$ is a trivial feasible solution of (\ref{detmaxeqproblem}),  $\bA_*$  ought to satisfy $|\det(\bA_*)|\geq\det(\mathbf{I})=1$. The inclusion of the MVIE of $\Pcal$ in $\text{conv}(\bS_g)$, due to the sufficient scattering condition (PMF.SS.i), and (\ref{inc1})
lead to $\bA_*(\mathcal{E}_\Pcal)\subset \Pcal$. We note that  $\bA_*(\mathcal{E}_\Pcal)$ is an ellipsoid in $\Pcal$, for which
\begin{eqnarray*}
\text{vol}(\bA_*(\mathcal{E}_\Pcal))=|\det(\bA_*)|\text{vol}(\mathcal{E}_P)\ge \text{vol}(\mathcal{E}_\Pcal),\end{eqnarray*}
which uses the lower bound $|\det(\bA_*)|\ge 1$. Conversely, based on the uniqueness of the MVIE for $\Pcal$, we can write $\text{vol}(\bA_*(\mathcal{E}_\Pcal))\le \text{vol}(\mathcal{E}_\Pcal)$. As a result,  we conclude that $|\det(\bA_*)|=1$ and
\begin{eqnarray}
\bA_*(\mathcal{E}_\Pcal)= \mathcal{E}_\Pcal. \label{eq:ellipsoid}
\end{eqnarray}
In other words, the constraint in (\ref{eq:detmaxeqconstr}) and the sufficient scattering condition (PMF.SS.i) together imply  $|\det(\bA_*)|=1$ and restrict $\bA_*$ to map $\mathcal{E}_\Pcal$ onto itself. Thus, we can write $\bA_*\bg_P=\bg_P$ for the center of $\mathcal{E}_\Pcal$.
As a consequence, according to Lemma \ref{lem:polaroftransform},  (\ref{eq:ellipsoid}) is identical to
\begin{eqnarray}
\bA_*^{-T}(\mathcal{E}_\mathcal{P}^{*,\bg_\mathcal{P}})= \mathcal{E}_\mathcal{P}^{*,\bg_\mathcal{P}}, \label{eq:ellipsoidpolar}
\end{eqnarray}
where $\mathcal{E}_\mathcal{P}^{*,\bg_\mathcal{P}}$ is also an ellipsoid. Therefore, (\ref{eq:ellipsoidpolar}) implies
\begin{eqnarray}
\bA_*^{-T}(\text{bd}(\mathcal{E}_\mathcal{P}^{*,\bg_\mathcal{P}}))= \text{bd}(\mathcal{E}_\mathcal{P}^{*,\bg_\mathcal{P}}). \label{eq:ellipsoidpolarbd}
\end{eqnarray}
Using the reversal of the set inclusion  property of the polar operation, given in Appendix A, and (\ref{inc1}), we can write $ \Pcal^{*,\bg_P}\subseteq \text{conv}(\bA_*\bS_g)^{*,\bg_P}$. Applying Lemma \ref{lem:polaroftransform}  to the right hand side of this expression, we obtain 
$\mathcal{P}^{*,\bg_\mathcal{P}}\subseteq\bA_*^{-T}(\text{conv}(\bS_g)^{*,\bg_\mathcal{P}})$,
which implies
\begin{eqnarray}
\bA_*^{-T}(\text{conv}(\bS_g)^{*,\bg_\mathcal{P}})\supseteq\text{ext}(\mathcal{P}^{*,\bg_\mathcal{P}}).\label{eq:extPsubset}
\end{eqnarray}
Based on the sufficient scattering condition  (PMF.SS.ii), we can write the inclusion expressions $\text{conv}(\bS_g)^{*,\bg_\mathcal{P}}\supseteq~\text{ext}(\mathcal{P}^{*,\bg_\mathcal{P}})$, $\text{bd}(\mathcal{E}_\mathcal{P}^{*,\bg_\mathcal{P}})\supseteq~\text{ext}(\mathcal{P}^{*,\bg_\mathcal{P}})$, and
\begin{eqnarray}
\bA_*^{-T}(\text{conv}(\bS_g)^{*,\bg_\mathcal{P}}) \cap \text{bd}(\mathcal{E}_\mathcal{P}^{*,\bg_\mathcal{P}})=\bA_*^{-T}(\text{ext}(\mathcal{P}^{*,\bg_\mathcal{P}})), \label{eq:part1}
\end{eqnarray}
where we also inserted (\ref{eq:ellipsoidpolarbd}). Furthermore, based on (\ref{eq:extPsubset}), we can write
\begin{eqnarray}
\bA^{-T}(\text{conv}(\bS_g)^{*,\bg_\mathcal{P}}) \cap \text{bd}(\mathcal{E}_\mathcal{P}^{*,\bg_\mathcal{P}})\supseteq\text{ext}(\mathcal{P}^{*,\bg_\mathcal{P}}).\label{eq:part2}
\end{eqnarray}
The expressions in (\ref{eq:part1}) and (\ref{eq:part2}) together imply
\begin{eqnarray}
\bA_*^{-T}(\text{ext}(\mathcal{P}^{*,\bg_\mathcal{P}}))\supseteq\text{ext}(\mathcal{P}^{*,\bg_\mathcal{P}}). \label{exttransform}
\end{eqnarray}
Since the cardinality of the sets on both sides are equal, we obtain
\begin{eqnarray}
\bA_*^{-T}(\text{ext}(\mathcal{P}^{*,\bg_\mathcal{P}}))=\text{ext}(\mathcal{P}^{*,\bg_\mathcal{P}}), \label{exttransform_eq}
\end{eqnarray}
which implies
$\bA_*^{-T}(\mathcal{P}^{*,\bg_\mathcal{P}})=\mathcal{P}^{*,\bg_\mathcal{P}}$. Invoking Lemma \ref{lem:polaroftransform}, we obtain $\bA_*(\Pcal)=\Pcal$ and deduce that $\bA_*(\text{ext}(\Pcal))=\text{ext}(\Pcal)$
due to the convexity of $\Pcal$. As a result, the condition that $\text{ext}(\Pcal)$ is a {\it permutation-and/or-sign-only invariant set}  implies $\bA_*=\bD\bPi$, with a diagonal sign matrix $\bD$ and a permutation matrix $\bPi$, which further implies the identifiability of the generative model.

The ``only if" part: Let $\bV_\Pcal\in \mathbb{R}^{r \times K}$ be a matrix that contains all $K$ vertices of $\Pcal$ in its columns. Clearly, the choice $\bS_g=\bV_\Pcal$ is  a sufficiently scattered factor for $\Pcal$. Suppose that  $\text{ext}(\Pcal)$ is not a permutation-and/or-sign-only invariant set, then there exists an $\bA\in\mathbb{R}^{r\times r}$, which is not a product of a diagonal and a permutation matrix, for which $\bA\bV_\Pcal=\bV_\Pcal\bPi$ for some permutation matrix $\bPi\neq \bI$. Since $|\det(\bA)|=1$,  $\bV_\Pcal\bPi$ would be another solution for the Det-Max optimization problem in (\ref{eq:detmaxoptimization}),  violating the identifiability of all sufficiently scattered sets.




%

\bibliographystyle{IEEEtran}
\bibliography{pmfbibfile}

%







\end{document}